\documentclass[conference]{IEEEtran}
\IEEEoverridecommandlockouts
\usepackage{cite}
\usepackage{mathtools}
\usepackage{times}
\usepackage[utf8]{inputenc}
\usepackage{caption}
\usepackage{hyperref}
\usepackage{graphicx}
\usepackage{booktabs}
\usepackage{subfigure}
\usepackage{color}
\usepackage{array}
\usepackage{tabularx}
\usepackage{changepage}
\usepackage{amssymb}% http://ctan.org/pkg/amssymb
\usepackage{pifont}% http://ctan.org/pkg/pifont
\usepackage{enumitem}
\usepackage{gensymb}
\usepackage{makecell}
\usepackage{textcomp}
\usepackage{float}
\usepackage{balance}
\usepackage{amsmath}
\usepackage[]{algorithm2e}
\usepackage{ctable} % for \specialrule command
\usepackage{amsthm}
\usepackage{hyperref}
\usepackage{amsmath,amssymb,amsfonts}
\newtheorem{theorem}{Theorem}
\usepackage{float}

\newcolumntype{L}[1]{>{\raggedright\let\newline\\\arraybackslash\hspace{0pt}}m{#1}}
\newcolumntype{C}[1]{>{\centering\let\newline\\\arraybackslash\hspace{0pt}}m{#1}}
\newcolumntype{R}[1]{>{\raggedleft\let\newline\\\arraybackslash\hspace{0pt}}m{#1}}

\newcommand{\RNum}[1]{\uppercase\expandafter{\romannumeral #1\relax}}
\AtBeginDocument{%
	\providecommand\BibTeX{{%
			\normalfont B\kern-0.5em{\scshape i\kern-0.25em b}\kern-0.8em\TeX}}}

\def\BibTeX{{\rm B\kern-.05em{\sc i\kern-.025em b}\kern-.08em
    T\kern-.1667em\lower.7ex\hbox{E}\kern-.125emX}}
\begin{document}

\title{SSDNet: State Space Decomposition Neural Network for Time Series Forecasting}

	\author{\IEEEauthorblockN{Yang Lin}
		\IEEEauthorblockA{\textit{School of Computer Science} \\
			\textit{The University of Sydney}\\
			Sydney, Australia \\
			ylin4015@uni.sydney.edu.au}
		\and
		\IEEEauthorblockN{Irena Koprinska}
		\IEEEauthorblockA{\textit{School of Computer Science} \\
			\textit{The University of Sydney}\\
			Sydney, Australia \\
			irena.koprinska@sydney.edu.au}
		\and
		\IEEEauthorblockN{Mashud Rana}
		\IEEEauthorblockA{\textit{Data61} \\
			\textit{CSIRO}\\
			Sydney, Australia \\
			mdmashud.rana@data61.csiro.au}
	}

\maketitle

\begin{abstract}
In this paper, we present SSDNet, a novel deep learning approach for time series forecasting.
SSDNet combines the Transformer architecture with state space models to provide probabilistic and interpretable forecasts, including trend and seasonality components and previous time steps important for the prediction. 
The Transformer architecture is used to learn the temporal patterns and estimate the parameters of the state space model directly and efficiently, without the need for Kalman filters.
We comprehensively evaluate the performance of SSDNet on five data sets, showing that SSDNet is an effective method in terms of accuracy and speed, outperforming state-of-the-art deep learning and statistical methods, and able to provide meaningful trend and seasonality components.

\end{abstract}

\section{Introduction}
\bstctlcite{IEEEexample:BSTcontrol}

Time series forecasting is an important task in many practical applications, e.g. predicting electricity demand, stock prices, immune response and disease progression over time.

Statistical methods such as ARIMA and exponential smoothing are well established for time series forecasting and are a part of the more general framework of State Space Models (SSMs) \cite{Durbin01book}. SSMs are considered interpretable models as components such as trend and seasonality can be extracted and used for decision making and explanation. However, they are not able to infer shared patterns from a set of related time series as each time series is fitted independently 
\cite{N-BEATS,Logsparse19NIPS}.

Deep learning methods have been recently investigated as a promising alternative, due to their ability to learn from raw data with minimum domain knowledge and to extract complex patterns, including shared patterns across related time series.  
Prominent examples include DeepAR \cite{DeepAR20}, a probabilistic forecasting model based on Long Short Term Memory (LSTM) neural networks, DeepSSM \cite{DeepSSM18NIPS} which combines SSM with LSTM,
Deep Factor models \cite{DeepFactor19} which combine LSTM with a local probabilistic model, and LogSparse Transformer \cite{Logsparse19NIPS} and Informer \cite{Informer20}, which are modifications of the Transformer architecture \cite{Transformer17NIPS} for time series forecasting. However, deep learning models are difficult to interpret; generating interpretable forecasts is crucial for adopting these systems in practical applications.

Approaches for producing interpretable forecasts using deep learning models have been recently proposed. For example, Guo et al. \cite{Tian19ICML} and Li et al. \cite{Li19IJCAI} explored the structure of LSTM and DeepSSM attempting to learn the variable importance for prediction. 
Oreshkin et al. \cite{N-BEATS} developed the deep learning architecture N-BEATS.
They showed that its interpretable configuration (N-BEATS-I) is able to successfully learn and output the trend and seasonality components of the forecast, which is useful for forecasting practitioners. 
However, N-BEATS is designed for univariate time series, it provides point forecasts and has limited ability to model trends with varying slope because it uses a deterministic trend model. 
In this paper, we present a new forecasting approach to address these issues.

The main contributions of this work are as follows:

\begin{enumerate}
	\item We present a new forecasting approach, called State Space Decomposition Neural Network (SSDNet), which combines the Transformer deep learning architecture with a SSM.
	SSDNet combines the advantages of deep learning (learning from raw data without intensive feature engineering and ability to infer shared patterns from related time series) with the interpretability of SSM models. It employs the Transformer architecture to learn the temporal pattern and directly estimate the parameters of SSM. To facilitate interpretability, we used a fixed form SSM to provide trend and seasonality components and the attention mechanism of the Transformer to identify which parts of the past history are most important for the prediction. 
	\item We evaluate the performance of SSDNet on five time series forecasting tasks. The results show that SSDNet achieved higher accuracy than the state-of-the-art deep learning models DeepAR, DeepSSM, LogSparse Transformer, Informer and N-BEATS and the statistical models SARIMAX and Prophet. SSDNet was also able to provide interpretable results by showing nonlinear trend and seasonality components. To evaluate the effectiveness of the fixed form SSM part, we conduct an ablation study by replacing the Transformer with LSTM, and this architecture also showed competitive results.
\end{enumerate}

\section{Problem Formulation}

We consider three tasks: 1) solar power forecasting, 2) electricity demand forecasting and 3) exchange rate forecasting. 
Accurate forecasting of the generated solar power and the electricity demand is needed for optimal scheduling of generators and large-scale integration of solar into the electricity grid. Over-prediction may lead to wasting energy, under-prediction may result in blackouts. Exchange rate forecasting is used to derive future monetary value, earn profits and avoid risks in international business environments.

\subsection{Data Sets}
\begin{figure}[!t]
\centering	  
\subfigure[]{\includegraphics[width=.48\columnwidth]{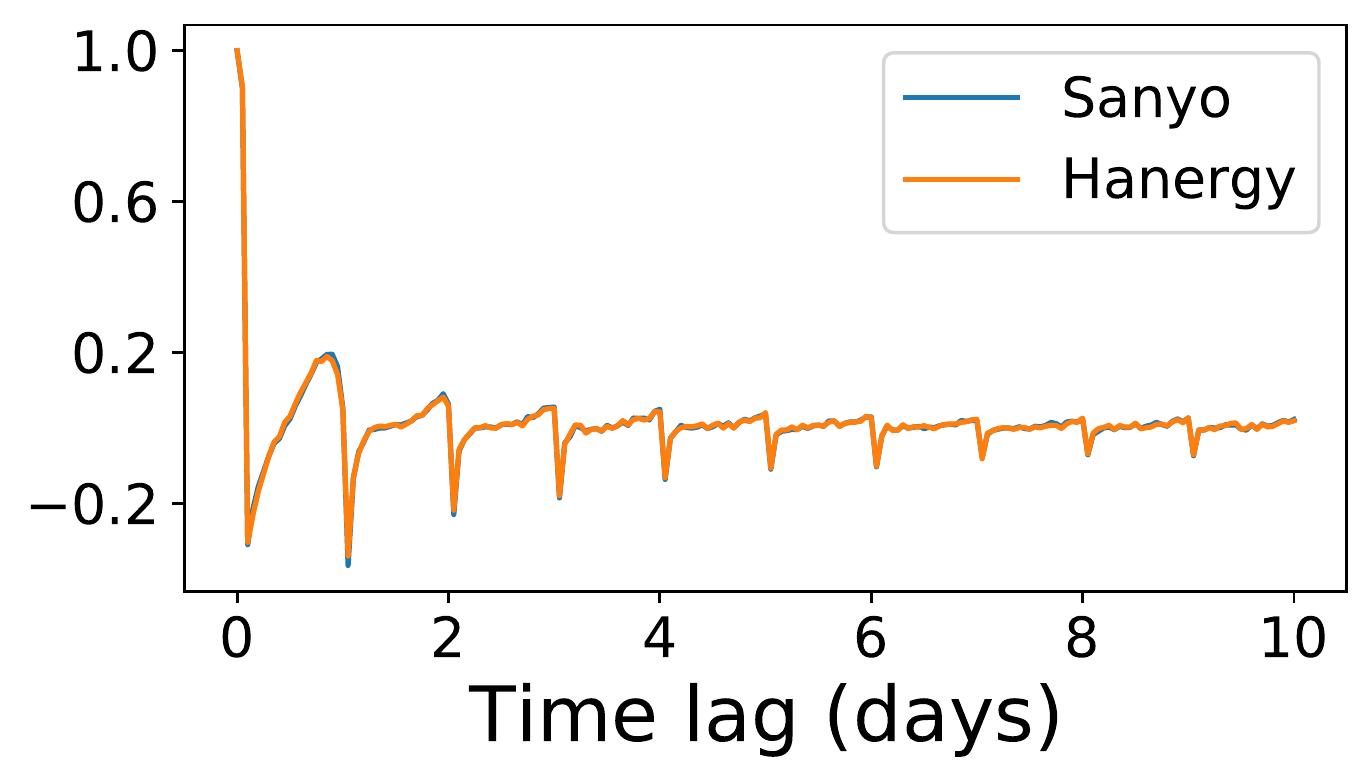}}
\subfigure[]{\includegraphics[width=.48\columnwidth]{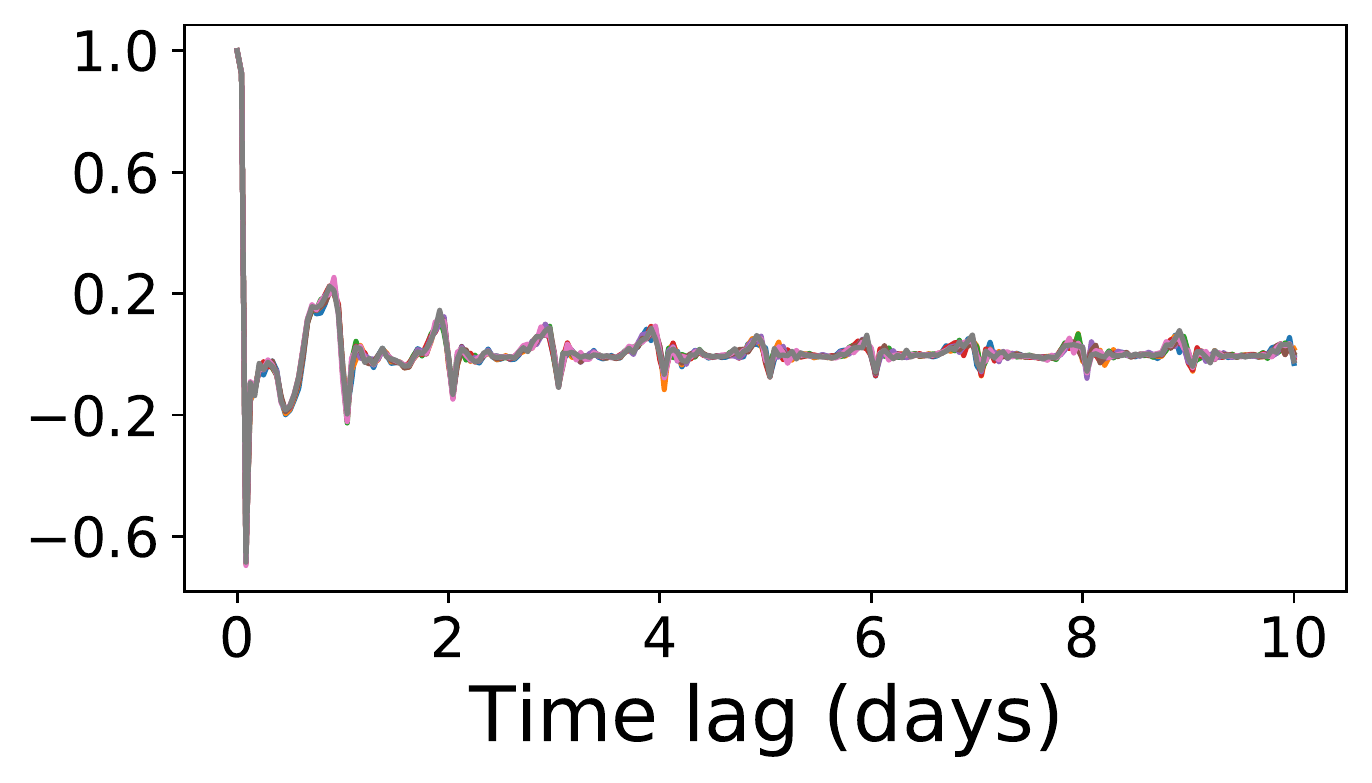}}
\subfigure[]{\includegraphics[width=.48\columnwidth]{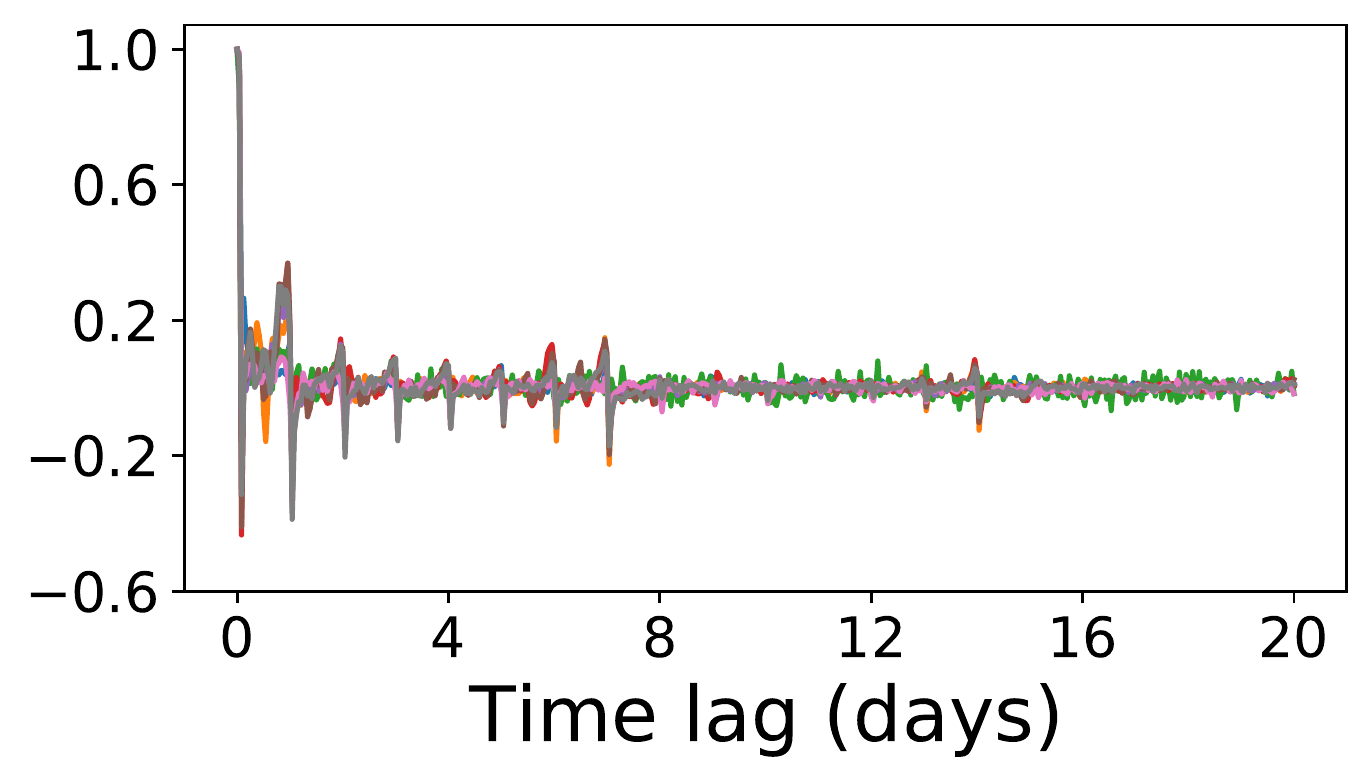}}		\subfigure[]{\includegraphics[width=.48\columnwidth]{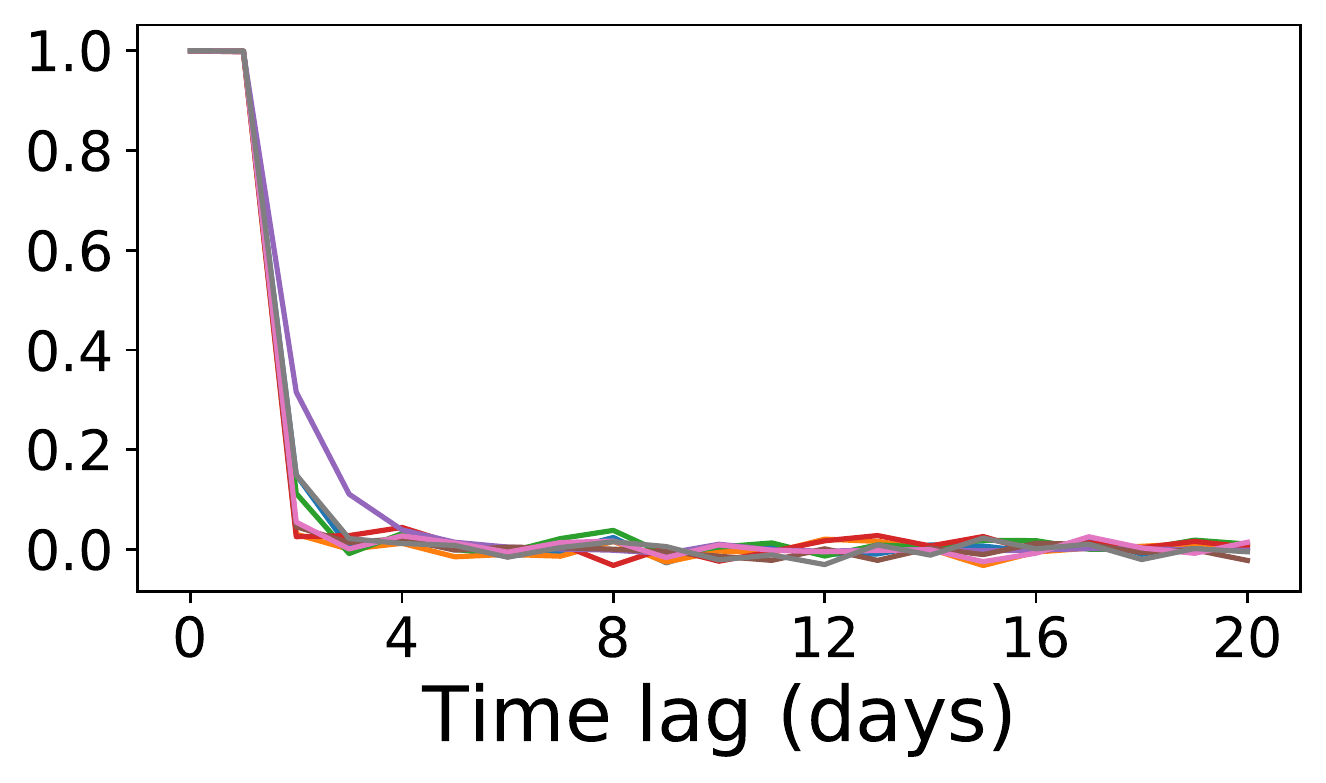}}	
\caption{Partial autocorrelation graphs for (a) Sanyo and Hanergy, (b) Solar, (c) Electricity and (d) Exchange data sets}
\label{pacf}
\end{figure}

\begin{figure}[!t]
\centering	  
\subfigure[]{\includegraphics[width=.48\columnwidth]{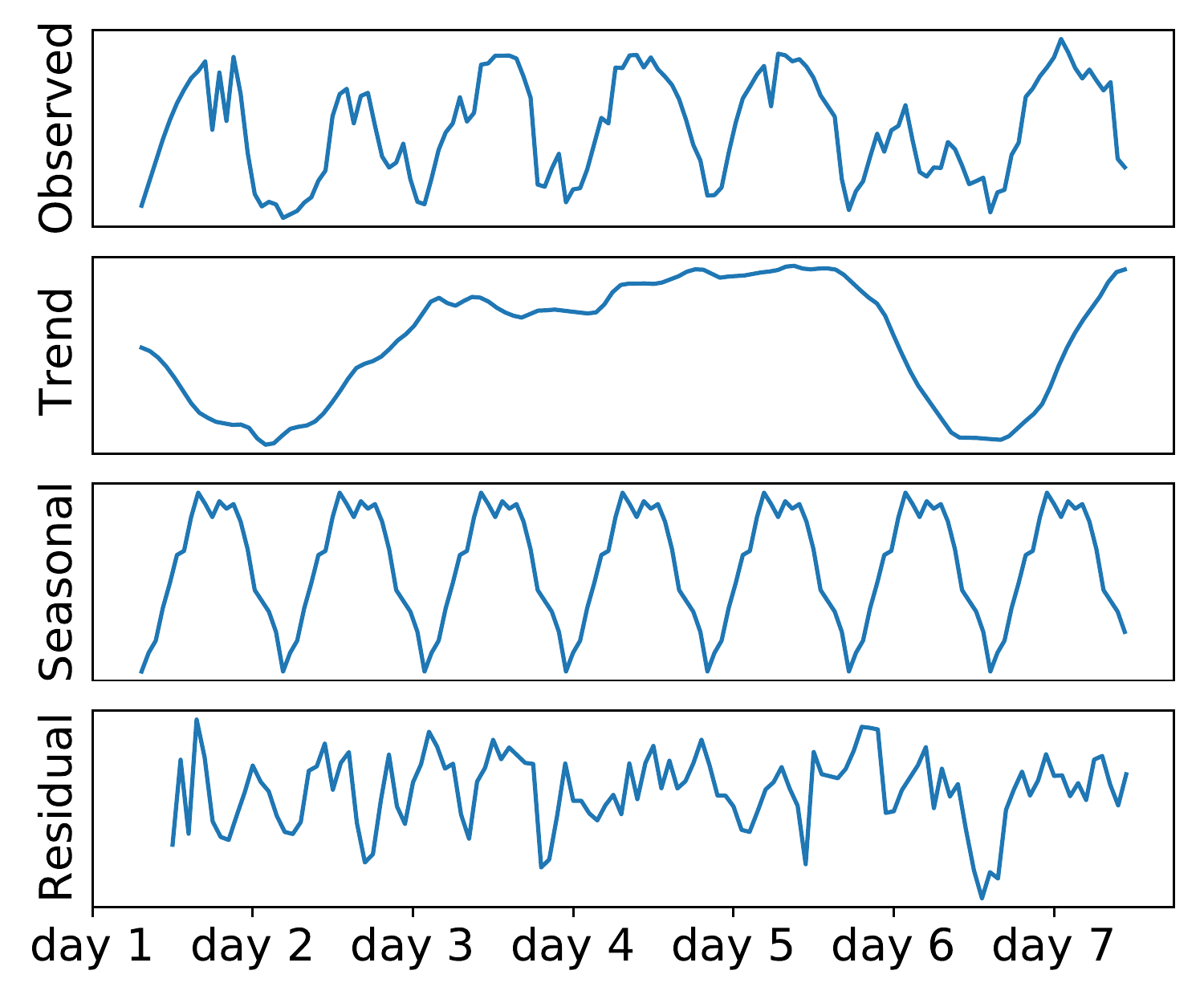}}
\subfigure[]{\includegraphics[width=.48\columnwidth]{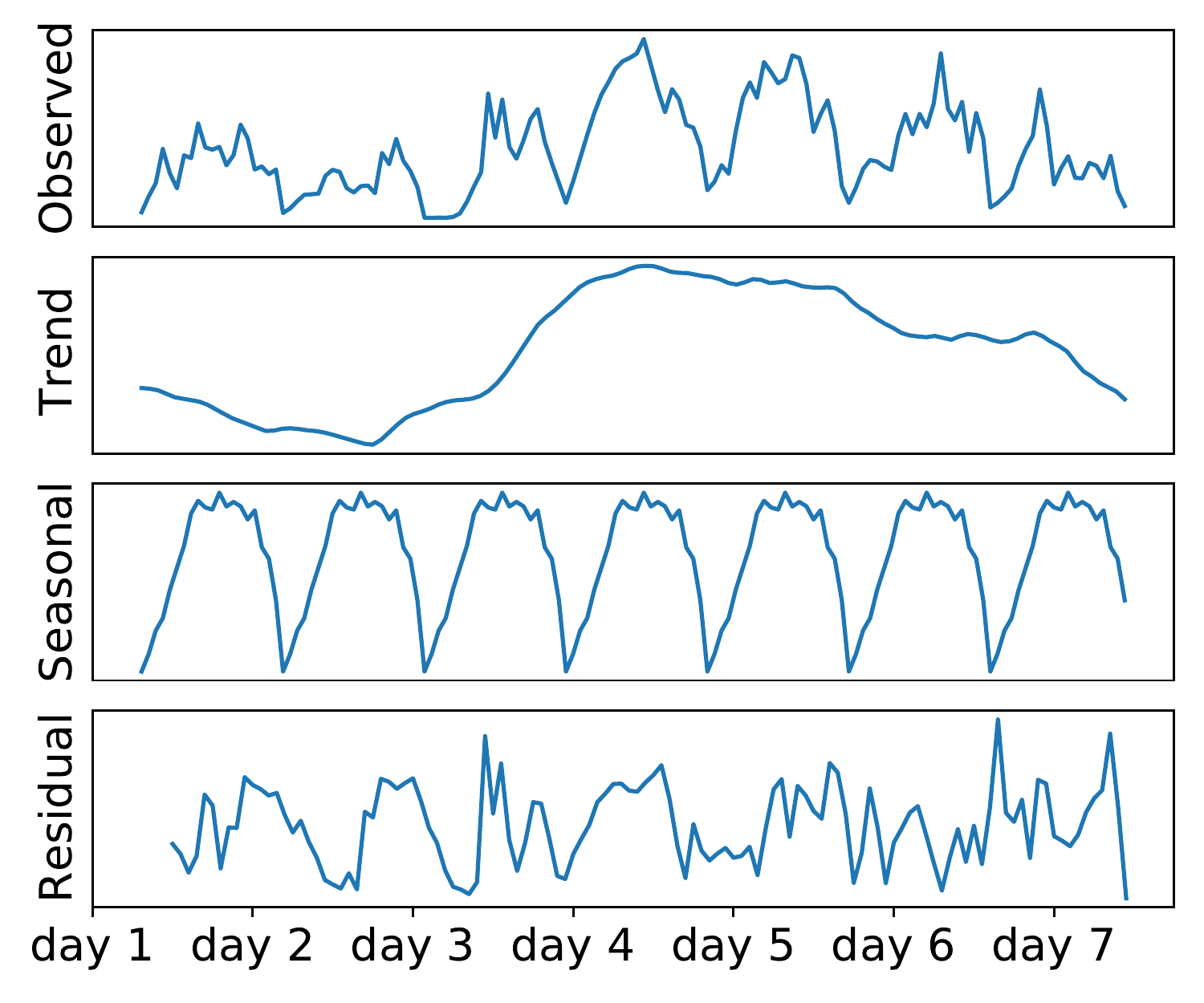}}		\subfigure[]{\includegraphics[width=.48\columnwidth]{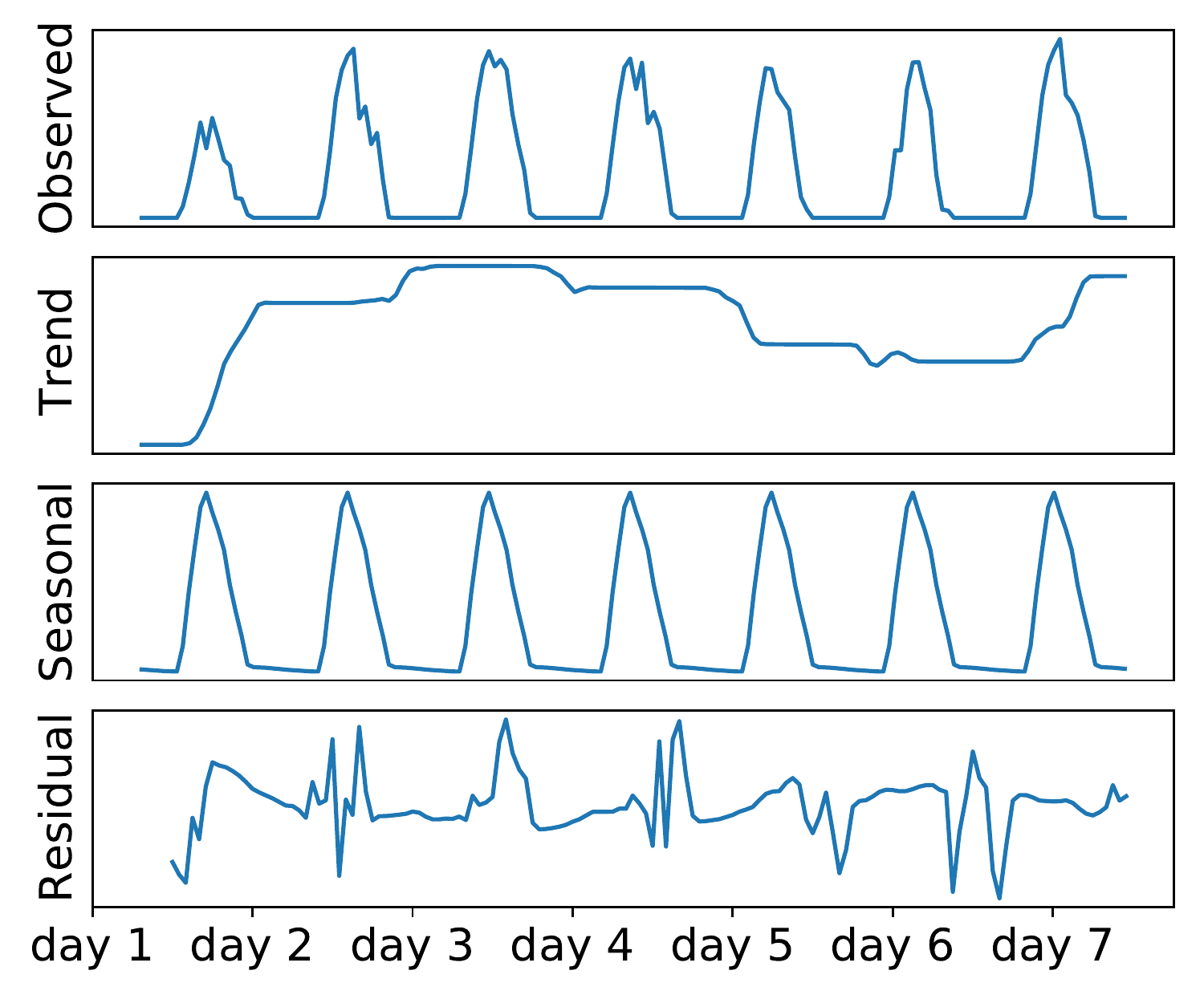}}
\subfigure[]{\includegraphics[width=.48\columnwidth]{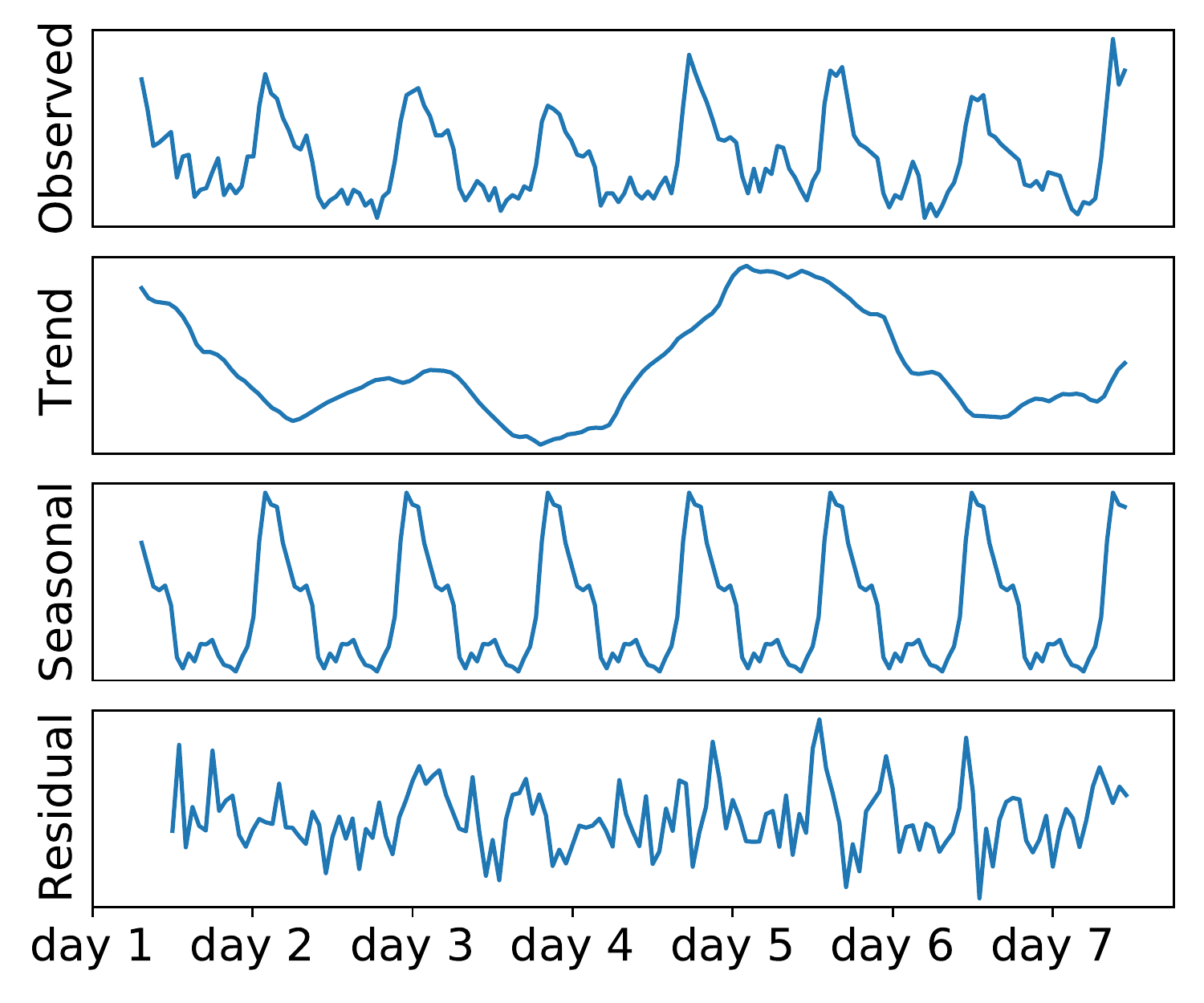}}		\subfigure[]{\includegraphics[width=.98\columnwidth]{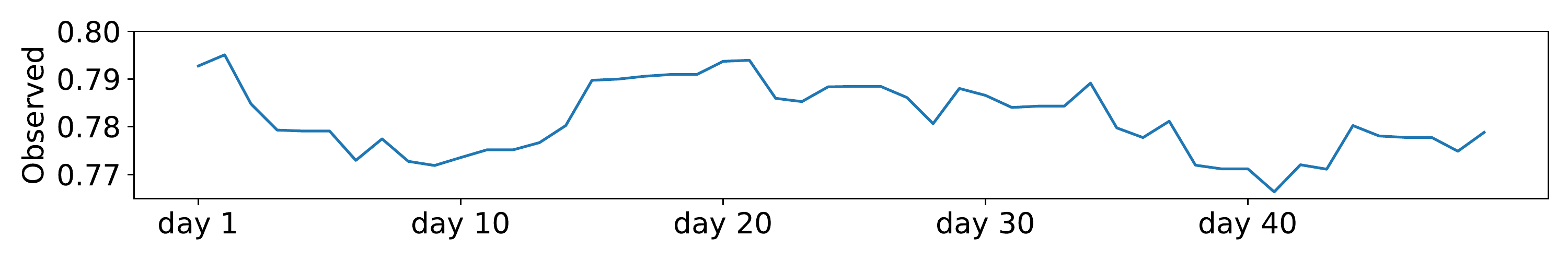}}	
\caption{Time series additive decomposition for (a) Sanyo, (b) Hanergy, (c) Solar, (d) Electricity data sets and (e) one series from Exchange data set}
\label{decomposition_results}
\end{figure}

\begin{table*}[!h]
	\renewcommand{\arraystretch}{1}	
	\centering
	\begin{tabular}{C{1.5cm} C{1.5cm} C{1.5cm} C{1.5cm} C{.8cm} C{.8cm} C{.8cm} C{.8cm} C{.8cm} C{.8cm} } 
		\specialrule{.1em}{.05em}{.05em} 
		&start date&end date &granularity &$L_d$&$N$& $n_{T}$  & $n_{C}$& $T_{l}$  & $T_{h}$ \\	
		\hline
		Sanyo&01/01/2011& 31/12/2016 & 30 minutes&20 & 1 &4 &3&20 &20\\				
		Hanergy&01/01/2011& 31/12/2017 & 30 minutes&20 & 1 &4 &3&20 &20\\			
		Solar&01/01/2006& 31/08/2006 & 1 hour&24 & 137 &0 &3&24 &24\\			
		Electricity&01/01/2011& 07/09/2014 & 1 hour&24 & 370 &0 &4&168 &24\\		
		Exchange&01/01/1990&31/12/2016&1 day&1&8&0&3&30 &20\\
		\specialrule{.1em}{.05em}{.05em} 
	\end{tabular}
	\caption{Dataset statistics. $L_d$ - number of steps per day, $N$ - number of series, $n_{T}$ - number of time-based features, $n_{C}$ - number of calendar features, $T_{l}$ - length of input series, $T_{h}$ - length of forecasting horizon. }
	\label{datasets} 
\end{table*}
\textbf{}

We use five data sets: Sanyo \cite{San}, Hanergy \cite{Han}, Solar \cite{Sol}, Electricity \cite{Ele} and Exchange \cite{Exc}; their statistics are summarised in Table \ref{datasets}.

\textbf{Sanyo} and \textbf{Hanergy} contain solar power generation data from two PV plants in Australia -  from 01/01/2011 to 31/12/2016 (6 years) for Hanergy and 01/01/2011 to 31/12/2017 (7 years) for Sanyo. Only the data between 7am and 5pm is considered and was aggregated at half-hourly intervals. For both datasets, weather and weather forecast data was also collected (see \cite{Yang20ICONIP} for more details) and used as covariates.
\textbf{Solar} contains solar power data from 137 PV plants in Alabama, USA, from 01/01/2006 to 31/08/2006.
\textbf{Electricity} contains electricity consumption data for 370 households from 01/01/2011 to 07/09/2014. The Solar and Electricity data is aggregated into 1-hour intervals. 
\textbf{Exchange} contains daily exchange rate data for 8 countries from 01/01/1990 to 31/12/2016. 
Following \cite{Logsparse19NIPS,Yang20ICONIP,TCAN_Yang}, calendar features are added according to the granularity of the datasets: Sanyo and Hanergy use \textit{month, hour-of-the-day and minute-of-the-hour}, Solar uses \textit{month, hour-of-the-day and age}, Electricity uses \textit{month, day-of-the-week, hour-of-the-day and age} and Exchange uses \textit{month, day-of-the-week and age}.

Fig. \ref{pacf} shows the partial autocorrelation plots for each data set. Fig. \ref{pacf} (a) shows the results for both  Sanyo and Hanergy since they have a similar partial autocorrelation pattern; both datasets include only 1 time series. 
Fig. \ref{pacf} (b), (c) and (d) show the results for Solar, Electricity and Exchange datasets respectively. All plots are for 8 time series for consistency since 8 is the common number of time series in the three datasets, see Table \ref{datasets}.   

Fig. \ref{decomposition_results} shows the additive decomposition results for a period of 7 days for Sanyo, Hanergy, Solar and Electricity. For Exchange we just show one time series as this dataset doesn't have a clear seasonality and additive decomposition methods cannot be applied.

We can see that all time series except Exchange exhibit seasonality and repetitive patterns with high partial autocorrelation, and sometimes there are also significant short-term trend variations. 
For the Exchange set, there is no clear repetitive pattern and its partial autocorrelation at lag 1 is the only significant one.

All data was normalized to have zero mean and unit variance.

\subsection{Problem Statement}

Given is: 1) a set of $N$ univariate time series (solar, electricity or exchange series) $\{\mathbf{Y}_{i,1:T_l}\}^N_{i=1}$, where $\mathbf{Y}_{i,1:T_l}\coloneqq[{y}_{i,1}, {y}_{i,2},...,{y}_{i,T_l}]$, $T_l$ is the input sequence length, and ${y}_{i,t}\in\Re$ is the value of the $i$th time series (generated PV solar power or consumed electricity) at time $t$; 2) a set of associated time-based multi-dimensional covariate vectors $\{\mathbf{X}_{i, 1: T_l+T_h}\}_{i=1}^{N}$, where $T_h$ denotes the length of the forecasting horizon. The covariates for the Sanyo and Hanergy datasets include: weather data $\{\mathbf{W1}_{i, 1: T_l}\}_{i=1}^{N}$, weather forecasts $\{\mathbf{WF}_{i, T_l+1: T_l+T_h}\}_{i=1}^{N}$ and calendar features $\{\mathbf{Z}_{i, 1: T_l+T_h}\}_{i=1}^{N}$, while the covariates for the Solar, Electricity and Exchange datasets include only calendar features.
Our goal is to predict the future values of the time series $\{\mathbf{Y}_{i,T_l+1:T_l+T_h}\}^N_{i=1}$, i.e. the PV power or electricity usage for the next $T_h$ time steps after $T_l$.

Specifically, SSDNet produces the probability distribution of the future values, given the past history:
\begin{equation}
\begin{split}
p\left(\mathbf{Y}_{i,T_l+1:T_l+T_h} \mid \mathbf{Y}_{i,1:T_l}, \mathbf{X}_{i, 1: T_l+T_h} ; \Phi\right)\\ = \prod_{t=T_{l}+1}^{T_{l}+T_{h}} p\left({y}_{i, t} \mid \mathbf{Y}_{i, 1: t-1}, \mathbf{X}_{i, 1: t} ; \Phi\right),
\end{split}
\end{equation}
where $\Phi$ denotes the parameters of SSDNet, and the input of SSDNet at step $t$ is the concatenation of ${y}_{i, t-1}$ and ${x}_{i, t}$. 
The models are applicable to all time series, so the subscript $i$ will be omitted in the rest of the paper for simplicity.

\section{Related Work}
DeepSSM \cite{DeepSSM18NIPS} combines SSMs with Recurrent Neural Networks (RNNs). A RNN is used to generate the parameters of a linear Gaussian SSM for time series forecasting and Kalman filter is used to derive posterior knowledge. 
DeepSSM is considered interpretable because it can recover the SSM parameters and they can be inspected. 
The Deep Factor model DF-LDS \cite{DeepFactor19} is a variation of DeepSSM that can handle non-Gaussian observations and it also uses Kalman filters.  

Although Kalman filters have been successfully used in linear dynamic systems for decades, the need for covariance matrix inversions is computationally expensive and limits their applicability for processing large data sets \cite{DeepKalman16arXiv,RKN19ICML}.  
In contrast, Eleftheriadis et al. \cite{Stefanos17NIPS} proposed Gaussian Process State Space Model (GPSSM) merging Gaussian process with a bi-directional RNN to approximate the posterior of a non-linear system. However, GPSSM has high complexity of inference ($O(T)$ per time step) which is a disadvantage for forecasting multi-horizon large data sets.

N-BEATS \cite{N-BEATS} is a novel deep neural architecture using backward and forward residual links and stacks of fully connected layers for univariate point forecasting.
Each N-BEATS layer contains a number of blocks which produce partial forecasts; these forecasts are aggregated at stack level and then network level in hierarchical fashion. The final forecast is the sum of the partial forecasts of all blocks.
N-BEATS can produce interpretable forecasting results without expert knowledge by employing a fixed form of block function to generate prediction with decomposition components.
The interpretable version (N-BEATS-I) uses the deterministic polynomial trend model
and assumes a stationary trend with constant slope. However, the local trend of time series over a short period (one or a few days) could vary significantly and inconsistently - e.g. the solar power generation or electricity consumption could decrease suddenly due to temporary cloud coverage or irregular activities of users (see Fig. \ref{decomposition_results}). Hence, a stochastic model which assumes a nonstationary trend with varying slope would be more suitable for such cases  \cite{Rob2019book} and this is what we use in SSDNet.

In summary, compared to DeepSSM, the proposed SSDNet: 
i) employs the Transformer architecture to learn temporal patterns, and estimates the parameters of SSM and the probability term directly instead of using Kalman filters; 
ii) uses the SSM with fixed and non-trainable transition matrix at the decoder to generate decomposition results within the forecasting horizon, while DeepSSM uses a more flexible SSM along all time steps; 
iii) combines Mean Absolute Error (MAE) and Negative Log-Likelihood (NLL) together as loss function to achieve accurate point and probabilistic forecasting results simultaneously instead of using NLL only; 
iv) shows better interpretability by using a stochastic decomposition technique in the form of SSM with an innovation term to model nonlinear trend and seasonality components and attention mappings to show the importance of historical time steps for forecasting. 
Compared to N-BEATS, SSDNet can model trend with nonlinear slope, provide probabilistic forecasts and also better utilizes the covariates.

\section{SSDNet}
The aim of SSDNet is to model temporal patterns effectively in order to provide accurate probabilistic and interpretable forecasts.

\subsection{Network Architecture}

Fig. \ref{SSDNet} illustrates the architecture of SSDNet.
SSDNet is based on the encoder-decoder framework and uses two key components: SSM and Transformer. SSM (in its fixed form) provides interpretable forecasting results, and the Transformer learns temporal patterns and estimates the SSM parameters from its decoder.

The feedforward process of SSDNet includes two steps. First, the Transformer of SSDNet processes the historical time steps and generates the latent components to estimate the parameters of SSM and the variance of the forecasted distribution. Second, SSM takes the state vector from the previous time step and uses it to predict the mean of the distribution. 
Unlike classical SSMs which decompose time series and estimate parameters via Kalman filter, we learn the SSM parameters using the Transformer architecture via attention mechanism and then combine the decomposition components (trend, seasonality and probability terms) in a form of SSM to produce the forecasting results. 

\begin{figure}[!t]
	\centering	  
	\includegraphics[width=\columnwidth]{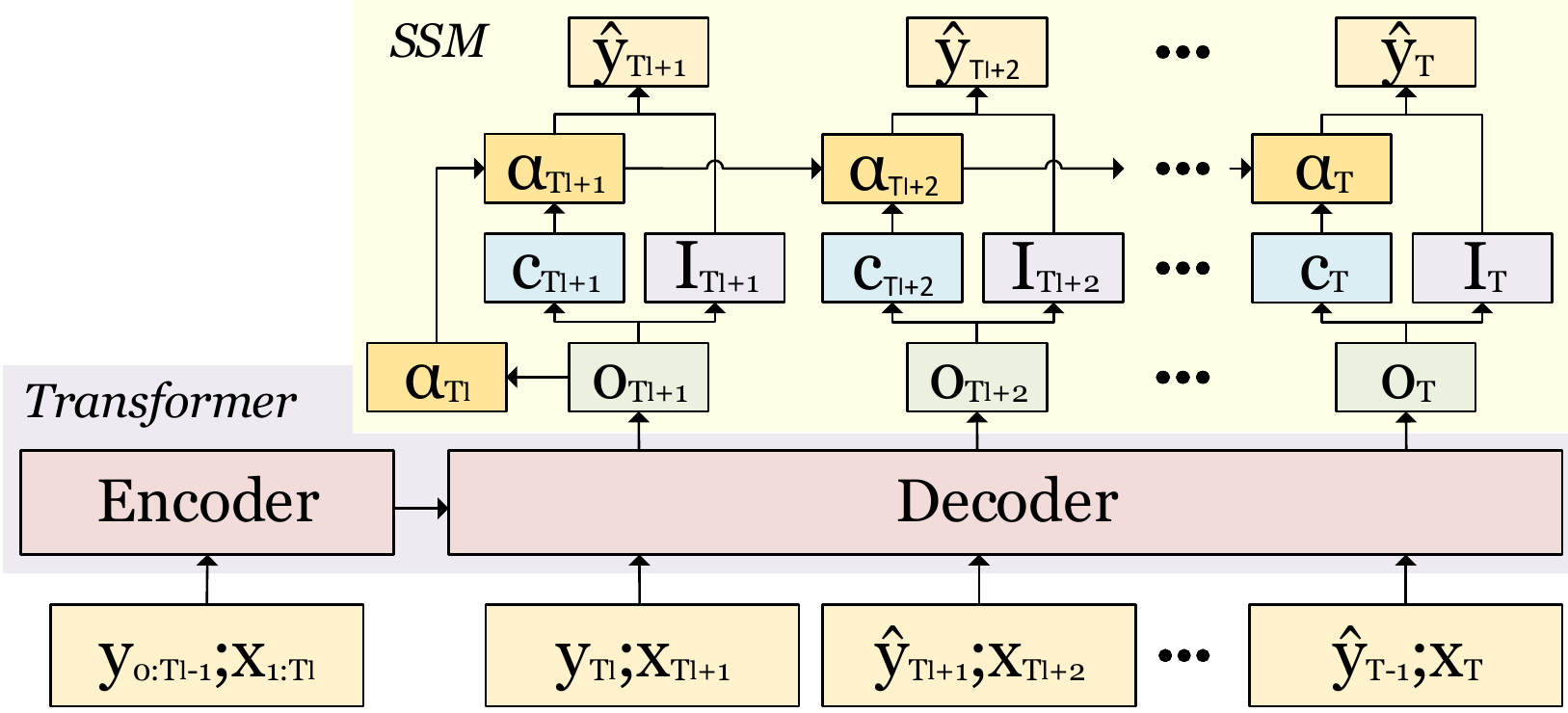}
	\caption{Structure of SSDNet}
	\label{SSDNet}
\end{figure}

As shown in Fig. \ref{SSDNet}, the Transformer extracts latent components $o_t$ from historical time series $y_{1:T_l},x_{1:T_t}$ in Eq. (\ref{o_Transformer}), and the latent components are further used to estimate the parameters of SSM. 
Note that the Transformer could be replaced by any sequence model to extract $o_t$ and learn the parameters of SSM. In Section VB, we evaluate the use of LSTM instead of Transformer in SSDNet; this modified architecture is called SSDNet-LSTM.

\begin{equation}
o_t = f(y_{1:T_l},x_{1:T_t})
\label{o_Transformer}
\end{equation}

Then, we employ additive time series decomposition model in the form of SSM and generate the prediction $\hat{y_{t}}$ at step $t$ by summing up the trend component $T_{t}$, seasonality component $S_{t}$ and probability component $I_{t}$. 
We decompose the times series using SSM and generate the forecasting result as:
\begin{equation}
\hat{y_{t}}=z_{t}^{T} \alpha_{t}+I_{t}, \quad t=1, \ldots, T_h
\label{observations}
\end{equation}
\begin{equation}
\alpha_{t+1}=\Gamma_{t} \alpha_{t}+c_{t}
\label{state_transition}
\end{equation}
\begin{equation}
I_{t} \sim \mathcal{N}(0,\sigma_{I_t}^2)
\label{I_Gaussian}
\end{equation}
where the latent state vector $\alpha_{t}\in \Re^{s\times 1}$ contains trend $Tr_{t}$ and seasonality $S_{t}$, $s$ is the number of seasonality and $I_t$ is sampled from the Gaussian distribution with zero mean and $\sigma_{I_t}^2$ variance. 

Different from the traditional SSMs, we remove the random noise term from the state transition in Eq. (\ref{state_transition}) to prevent the state noise propagation and estimate the probabilistic term $I_{t}$ directly as shown in Eq. (\ref{observations}). 
The SSM part of SSDNet does not process historical series directly but instead utilizes the latent component generated by the Transformer and encodes the information of the time steps before the forecasting horizon in the initial state vector $\alpha_{T_l}$. We also introduce the innovation term $c_t\in \Re^{s\times 1}$ in the state transition Eq. (\ref{state_transition}) to allow SSDNet to learn stochastic trends with fluctuations in time series. This is necessary as the simple summation of linear trend and seasonality components cannot model complex series effectively and would introduce large residuals (see Fig. \ref{decomposition_results}). 
While the traditional SSMs need to use Kalman filter to update the posterior and derive the covariance matrix to optimize noise terms, SSDNet uses the Transformer architecture resulting in less matrix computation and memory usage.

The innovation term $c_t$ and the variance $\sigma_{I_t}^2$ are learnt from the latent factor $o_t$ directly, as illustrated in Eq. (\ref{I_model}) and (\ref{c_model}). The use of the Softplus function ensures that SSDNet always generates positive variance, and the HardSigmoid function \cite{HardSigmoid15NIPS} is used for a speed-up:
\begin{equation}
\begin{split}
\sigma_{I_{t}}^2 =g_s(o_t)
=&\text{Softplus}(\text{Linear}(o_t))	\\	
=&\log(1+\exp(\text{Linear}(o_t)))
\end{split}
\label{I_model}
\end{equation}
\begin{equation}
\begin{split}
c_{t} =g_c(o_t)
=&\text{HardSigmoid}(\text{Linear}(o_t))-0.5\\
=&\left\{\begin{array}{ll}
-0.5 & \text { if } \mathrm{x} \leq-3 \\
0.5 & \text { if } \mathrm{x} \geq+3 \\
\mathrm{\text{Linear}(o_t)} / 6 & \text { otherwise }
\end{array}\right.
\end{split}
\label{c_model}
\end{equation}

The parameters state vector $\alpha_{t}$, state transition matrix $\Gamma_t$ and result transition matrix $z_t$ that determine the structure of SSM in Eq. (\ref{observations}) and (\ref{state_transition}) are shown below:
\begin{equation}
\alpha_{t}=\left(\begin{array}{c}
Tr_{t} \\
S_{1:s-1, t} \\
\end{array}\right),
{z}_{t}=\left(\begin{array}{l}
1 \\
1 \\
0_{s-2}
\end{array}\right),
\end{equation}
\begin{align*}
\Gamma_{t}=\left(\begin{array}{cccc}
1  & 0_{s-2}^{\prime} & 0 \\
0  & -1_{s-2}^{\prime} & -1 \\
0_{s-2}  & I_{s-2} & 0_{s-2} 
\end{array}\right)
\end{align*}

Note that $\Gamma_t$ and $z_t$ are non-trainable and fixed for all time steps. In this work, we assume that the trend follows a random walk process, and other processes such as moving average could be modelled by adjusting the first row of the transition matrix $\Gamma_t$. The initial trend $Tr_{0}$ and seasonality $S_{-s+1},..., S_{-1}, S_{0}$ values of the state vector $\alpha_0$ are unknown and need to be generated based on the latent components, they are learned from historical steps:
\begin{equation}
\begin{split}
\alpha_{0} &=g_c(o_{T_{l+1}})=\text{HardSigmoid}(\text{Linear}(o_{T_{l+1}}))-0.5
\end{split}
\label{alpha_model}
\end{equation}

Alternatively, the forecasting result (\ref{observations}) could be written as the aggregation of the predicted trend, seasonality and variance to provide probabilistic result:
\begin{equation}
\hat{y_{t}}\sim \mathcal{N}(Tr_t+S_t,\,\sigma_{I_t}^2)\,
\end{equation}
where we consider the real-world data follows a Gaussian distribution (other distribution function could be used in Eq. (\ref{I_Gaussian}) for higher flexibility), and predictions are sampled from the distribution. The $\rho$-quantile output could be generated via the inverse cumulative probability distribution: $\hat{y_{t}}=F_{t}^{-1}(\rho)$.

At step $t$, Eq. (\ref{state_transition}) could be written in an additive form, where the trend and seasonality evolve by adding the innovation term at each time step, as shown in Eq. (\ref{trend}) and (\ref{seasonality}). The trend and seasonality models are based on the classic random walk and dummy seasonal, and the change of trend and seasonality are dominated by the innovation variable. The innovation variable for trend can be different at every time step and allow SSDNet to model the nonstationary trend with varying slope. For the extreme case when all innovation terms are zero, the trend would be constant and the seasonality over all periods will be the same. 
\begin{equation}
Tr_t= Tr_{t-1}+ c_t[1]
\label{trend}
\end{equation}
\begin{equation}
S_t=-\sum_{j=1}^{s-1}S_{t-j}+ c_t[2]
\label{seasonality}
\end{equation}

\begin{theorem}
	$\forall t>0$ and $t\in\mathbb{Z}$, the range of trend $Tr_t$ and seasonality $S_t$ of SSDNet are bounded: $
	Tr_t\subset[-(t+1)\times0.5, (t+1)\times0.5]$, $S_t\subset[-(s-1+t)\times0.5, (s-1+t)\times0.5]$. 
	\label{SSDNetbounds}
\end{theorem}

\begin{proof}
	Assume that innovation term $c_t$, trend and seasonality have maximum values: 1)  $c_t$ outputs its maximum value of 0.5 at all time steps, 
	2) the initial trend has its maximum value $Tr_0=0.5$ and 
	3) the seasonality has its maximum value of $S_0=-\sum_{j=1}^{s-1} -0.5=0.5\times(s-1)$.
	
	Following Eq. (\ref{trend}) and (\ref{seasonality}), the trend and seasonality increase using the innovation terms with maximum values of $c_t[1]=0.5$ and $c_t[2]=0.5$ respectively at each time step.
	
	Considering the initial values and increments together, the maximum values of the trend and seasonality are $(t+1)\times0.5$ and $(s-1+t)\times0.5$ at time step $t$, correspondingly.  
	Since $g_c(.)$ is a symmetric activation function, its lower bound is the negative value of its upper bound, hence the minimum values of trend and seasonality are the negative of their maximum values. 
\end{proof}

Theorem \ref{SSDNetbounds} implies that the upper and lower bound of the SSDNet trend and seasonality components expand in the time domain and SSDNet can model time series where the variation in the seasonal pattern increases over time. The variation of the trend component is significantly smaller than that of the seasonality and helps SSDNet to model a slow-moving trend and significantly varying seasonality.
The constraints of both components help SSDNet to learn an inherent structure and allow $\alpha_t$ to generate meaningful waveform (observable in our case studies) and enable the interpretability of forecasting results.	

In our case since the data is normalized to have zero mean and unit variance, when visualizing trend and seasonality components, the bias (data mean) is added to the trend and both the trend and seasonality are scaled (data variance) back.
\subsection{Loss Function}

SSDNet needs to provide both accurate point and probabilistic forecasts. 
Accordingly, we developed the loss function shown in  Eq. (\ref{loss}), which considers both the point and probabilistic forecasts by combining the MAE and NLL using the regularization parameter $a$: 
higher $a$ allows the model to pay more attention to the probabilistic forecast. 
The point forecast is $\hat{y}_t$ and it consists of trend and seasonality. The parameters of SSDNet are optimized by minimizing this function.

\begin{equation}
\begin{split}
&L(\hat{y}_{T_l+1:T},\sigma_{I_{T_l+1:T}}^2,y_{T_l+1:T},a)\\=&~ a\times\operatorname{NLL}(\hat{y}_{T_l+1:T},\sigma_{I_{T_l+1:T}}^2,y_{T_l+1:T})\\
&+
\operatorname{MAE}(\hat{y}_{T_l+1:T},y_{T_l+1:T}) \\
=&-\frac{a}{2T_h}\times\Big(T_h \log (2 \pi)+ \sum_{t=T_l+1}^{T} \log \left|\sigma_{I_t}^2\right| \\
&+\sum_{t=T_l+1}^{T} (y_t-\hat{y}_t)^{2} \sigma_{I_t}^{-2} \Big)
+\frac{1}{T_h}\sum_{t=T_l+1}^{T} \left|y_t-\hat{y}_t\right|
\end{split}
\label{loss}
\end{equation}

\section{Experiments}
We compare the performance of SSDNet with nine other models: six state-of-the-art autoregressive deep learning models (DeepAR, DeepSSM, N-BEATS-G, N-BEATS-I, LogSparse Transformer and Informer), a SSM (SARIMAX), an interpretable regression model (Prophet) and a persistence model: 
\begin{itemize}
	\item Persistence is a typical baseline in forecasting and considers the time series of the previous day as the prediction for the next day. For the Exchange dataset we use the last 20 steps of the input sequence.
	\item SARIMAX \cite{Durbin01book} is an extension of the ARIMA and can handle seasonality with exogenous factors.
	\item Prophet \cite{Prophet18} is an interpretable regression model predicting trend, seasonality and holiday components.
	\item DeepAR \cite{DeepAR20} is a widely used sequence-to-sequence probabilistic forecasting model. 
	\item DeepSSM \cite{DeepSSM18NIPS} fuses SSM with RNNs to incorporate structural assumptions and learn complex patterns from the time series. It is the state-of-the-art deep forecasting model that employs SSM.
	\item N-BEATS \cite{N-BEATS} is based on backward and forward residual links and stacks of fully connected layers. N-BEATS-G provides generic forecasting results, while N-BEATS-I provides interpretable results by decomposing the time series into trend and seasonality. We introduced covariates to N-BEATS at the input of each block to facilitate multivariate series forecasting. 
	\item LogSparse Transformer \cite{Logsparse19NIPS} is a recently proposed variation of the Transformer architecture for time series forecasting with convolutional attention and sparse attention; it is denoted as "LogTrans" in Table \ref{Accuracy}. 
	\item Informer \cite{Informer20} is a Transformer-based forecasting model based on the ProbSparse self-attention and self-attention distilling. We modified the Informer to provide probabilistic results by using Eq. (\ref{I_model}) as the output layer.
\end{itemize}

\begin{table}[!t]
	\renewcommand{\arraystretch}{1}	
	\centering
	\begin{tabular}{C{1.5cm} C{.8cm} C{.6cm} C{.7cm} C{.6cm} C{.8cm} C{.6cm} } 
		\specialrule{.1em}{.05em}{.05em} 
		&$\lambda$& $\delta$ & $d_{hid}$ & $n_{l}$ & $d_{k\&v}$ &$n_{h}$\\	
		\hline
		Sanyo&0.005& 0 & 12 & 2 &6 &2\\				
		Hanergy&0.005& 0 & 16 & 3 &6 &3\\			
		Solar&0.005& 0.1 & 16 & 3 &6 &3\\			
		Electricity&0.001& 0.1 & 24 & 3 &8 &2\\		
		Exchange&0.005&0&12&2&4&3\\
		\specialrule{.1em}{.05em}{.05em} 
	\end{tabular}
	\caption{Hyperparameters for SSDNet}
	\label{SSDNet_Parameters} 
\end{table}
\begin{table*}[!ht]
	\renewcommand{\arraystretch}{1}	
	\centering
	\begin{tabular}{ C{2cm}   C{2.2cm}  C{2.2cm}  C{2.2cm}  C{2.2cm} C{2.2cm} } 
		\specialrule{.1em}{.05em}{.05em} 
		& Sanyo& Hanery& Solar& Electricity& Exchange\\
		\hline
		Persistence &0.154/-&0.242/-&0.256/-&0.069/-&0.016/-\\ 		
		SARIMAX&0.124/0.096&0.145/0.098&0.256/0.192&0.196/0.079&\textbf{0.010}/\textbf{0.006}\\ 
		Prophet&0.104/0.054&0.152/0.079&0.268/0.169&0.112/0.055&0.017/0.013\\ 
		DeepAR
		&0.070/0.031&0.092/0.045&0.222$^\diamond$/0.093$^\diamond$&0.075$^\diamond$/0.040$^\diamond$&0.014/0.009\\
		DeepSSM
		&0.042/0.023&0.070/0.053&0.223/0.181&0.083$^\diamond$/0.056$^\diamond$&0.014/0.012\\
		LogTrans 
		&0.067/0.036&0.088/0.047&0.210$^\diamond$/0.082$^\diamond$&\textbf{0.059}$^\diamond$/{0.034}$^\diamond$&0.017/0.008\\
		Informer
		&0.046/0.022&0.084/0.046&0.215/0.115&0.068/\textbf{0.033}&0.014/0.009\\
		N-BEATS-I&0.091/&0.154/-    &0.215/-&0.102/-&0.014/-\\
		N-BEATS-G&0.077/-&0.132/-&0.212/-&0.061/-&0.018/-\\			
		\hline
		SSDNet&\textbf{0.040}/\textbf{0.020}&\textbf{0.059}/\textbf{0.032}&\textbf{0.209}/\textbf{0.074}&0.068/\textbf{0.033}&0.013/\textbf{0.006}\\	
		\specialrule{.1em}{.05em}{.05em} 
	\end{tabular}
	\caption{$\rho$0.5/$\rho$0.9-loss of data sets with various granularities. $\diamond$ denotes results from \protect\cite{Logsparse19NIPS}.}
	\label{Accuracy} 
\end{table*}
All models were implemented with PyTorch 1.6 on Tesla T4 16GB GPU under Linux environment.
The deep learning models were optimized by mini-batch gradient descent with the Adam optimizer and a maximum number of epochs 200. We used Bayesian optimization for hyperparameter search for all deep learning models with  maximum number of 20 iterations. The hyperparameters with a minimum loss on the validation set were selected. The models used for comparison were tuned based on the authors' recommendations. The probabilistic forecasting models use the NLL loss, and the point forecasting model (N-BEATS) uses the mean squared loss.

Following the experimental setup in \cite{Logsparse19NIPS}, \cite{Yang20ICONIP} and \cite{TCAN_Yang}, we used the following training, validation and test split: 	for Sanyo and Hanergy - the data from the last year as test set, the second last year as validation set for early stopping and the remaining data (5 years for Sanyo and 4 years for Hanergy) as training set; for Solar and Electricity - the last week data as test set (from 25/08/2006 for Solar and 01/09/2014 for Electricity), the week before as validation set and the remaining data as training set. 
For the Exchange dataset, we used the 480 working days from 02/02/2015 is test set, the 480 working days before as validation set and the remaining data as training set. 
For all data sets, the data preceding the validation set is split in the same way into three subsets and the corresponding validation set is used to select the best hyperparameters.

For the Transformer-based models, we used learnable position and ID (for Solar, Electricity and Exchange sets) embedding. 
For SSDNet, the seasonality $s$ was set to 20 for Sanyo, Hanergy and Exchanges and 24 for Solar and Electricity, the loss function regularization parameter $a$ was fixed as 0.5, the learning rate $\lambda$ was fixed, the dropout rate $\delta$ was chosen from \{0, 0.1, 0.2\}, the hidden layer dimension size $d_{hid}$ and number of layers $n_{l}$ were chosen from \{8, 12, 16, 24, 32\} and \{2, 3, 4\}, the query and value's dimension size $d_{k\&v}$ and number of heads $n_{h}$ were chosen from \{4, 6, 8, 12\} and \{2, 3, 4\}.  	
The selected best hyperparameters for SSDNet are listed in Table \ref{SSDNet_Parameters} and used for the evaluation of the test set.

Following \cite{DeepAR20,DeepSSM18NIPS,Logsparse19NIPS,Yang20ICONIP,TCAN_Yang}, we report the standard $\rho$0.5 and $\rho$0.9-quantile losses. Note that $\rho$0.5 is equivalent to the Mean Absolute Percentage Error (MAPE) \cite{TRMF16NIPS}. Given the ground truth $y$ and $\rho$-quantile of the predicted distribution $\hat{y}$, the $\rho$-quantile loss is given by:
\begin{equation}
\begin{split}
\mathrm{QL}_{\rho}(y, \hat{y})&=\frac{2\times\sum_{t} P_{\rho}\left(y_{t}, \hat{y}_{t}\right)}{\sum_{t}\left|y_{t}\right|}, 
\\
\quad P_{\rho}(y, \hat{y})&=\left\{\begin{array}{ll}
\rho(y-\hat{y}) & \text { if } y>\hat{y} \\
(1-\rho)(\hat{y}-y) & \text { otherwise }
\end{array}\right.
\end{split}
\end{equation}

\subsection{Accuracy Analysis}
The $\rho$0.5 and $\rho$0.9 losses are shown in Table \ref{Accuracy}. As N-BEATS and Persistence do not produce probabilistic forecasts, only the $\rho$0.5-loss is reported for them. We can see that overall SSDNet is the most accurate model - it outperforms all the other methods on all data sets except Electricity and Exchange. For the Electricity data set, SSDNet has the best $\rho$0.9-loss and second-best $\rho$0.5-loss among the probabilistic forecasting models and is also the best among the interpretable models (N-BEATS-I, DeepSSM and Prophet).  
For the Exchange data set, SSDNet is the second-best model after SARIMAX. 

SSDNet performs better than DeepSSM on all datasets which indicates that the proposed new architecture, loss function and the use of fixed SSM were beneficial. 

The two N-BEATS models perform well on Solar and Electricity (univariate data sets) but much worse on Sanyo and Hanergy (multivariate, including the weather features); a possible reason is that N-BEATS was designed for univariate forecasting and our modification to handle multivariate datasets was not sufficiently effective in this case. 

We further investigated the performance on the Exchange dataset, for which the deep neural network models, including the proposed SSDNet, performed worse than SARIMAX. This finding is consistent with \cite{Lai18SIGIR}, and the reason is that Exchange is the only data set without repetitive patterns and a seasonal component. 
Fig. \ref{Visualized_results_exchange} and Fig. \ref{Visualized_results_SARIMA} show the predictions of SSDNet and SARIMAX for 2 test samples from the Exchange data set. We can see that SSDNet tends to generate predictions with a smooth trend and some variations, while SARIMAX tends to generate predictions that do not vary much with time. Hence, the success of SARIMAX on the Exchange data set can be explained with the characteristics of this data set -  it is more stable and does not have strong repetitive patterns and seasonality, as shown in Section IIA and Fig. \ref{pacf}.

Another interesting finding for Exchange is that when we simply use the last time step value $\mathbf{Y}_{T_l}$ of the input series as the predictions $\mathbf{Y}_{T_l+1:T_l+T_h}$, we obtain $\rho$0.5-loss of 0.010 which is as good as SARIMAX and outperforms all other models.
However, SSDNet and the other deep learning models demonstrate significantly higher accuracy than SARIMAX on the other four data sets which include repetitive patterns.

\begin{table*}[!ht]
	\renewcommand{\arraystretch}{1}	
	\centering
	\begin{tabular}{ C{2cm}   C{2.2cm}  C{2.2cm}  C{2.2cm}  C{2.2cm} C{2.2cm} } 
		\specialrule{.1em}{.05em}{.05em} 
		& Sanyo& Hanery& Solar& Electricity& Exchange\\
		\hline
		DeepAR 
		&0.070/0.031&0.092/0.045&0.222$^\diamond$/0.093$^\diamond$&0.075$^\diamond$/0.040$^\diamond$&0.014/0.009\\		
		DeepSSM 
		&0.042/0.023&0.070/0.053&0.223/0.181&0.083$^\diamond$/0.056$^\diamond$&0.014/0.012\\	
		SSDNet&\textbf{0.040}/\textbf{0.020}&\textbf{0.059}/\textbf{0.032}&{0.209}/\textbf{0.074}&\textbf{0.068}/\textbf{0.033}&0.013/\textbf{0.006}\\	
		SSDNet-LSTM&\textbf{0.040}/\textbf{0.020}&0.066/0.037&\textbf{0.205}/0.075&0.071/0.037&\textbf{0.012}/\textbf{0.006}\\
		\specialrule{.1em}{.05em}{.05em} 
	\end{tabular}
	\caption{Ablation study - $\rho$0.5/$\rho$0.9-loss. $\diamond$ denotes results from \protect\cite{Logsparse19NIPS}.}
	\label{Ablation_Accuracy} 
\end{table*}
\begin{table*}[!ht]
	\renewcommand{\arraystretch}{1}	
	\centering
	\begin{tabular}{ C{2cm}   C{2.2cm}  C{2.2cm}  C{2.2cm}  C{2.2cm} } 
		\specialrule{.1em}{.05em}{.05em} 
		& Sanyo& Hanery& Solar& Exchange\\
		\hline
		DeepSSM&25163$\pm$201&9750$\pm$135&224343$\pm$692&71188$\pm$343\\	
		SSDNet&6006$\pm$91&3081$\pm$91&61720$\pm$206&13817$\pm$62\\	
		SSDNet-LSTM&5153$\pm$68&2079$\pm$29&56989$\pm$127&7890$\pm$99\\	
		\hline
	\end{tabular}
	\caption{Training time per epoch (milliseconds) - comparison of DeepSSM and SSDNets on Sanyo, Hanergy, Solar and Exchange data sets.}
	\label{train_speed} 
\end{table*}
\begin{table*}[!ht]
	\renewcommand{\arraystretch}{1}	
	\centering
	\begin{tabular}{ C{2cm}   C{2.2cm}  C{2.2cm}  C{2.2cm}  C{2.2cm} } 
		\specialrule{.1em}{.05em}{.05em} 
		& Sanyo& Hanery& Solar& Exchange\\
		\hline
		DeepSSM&242$\pm$54&233$\pm$6&278$\pm$9&250$\pm$7\\	
		SSDNet&411$\pm$177&438$\pm$132&573$\pm$215&389$\pm$125\\	
		SSDNet-LSTM&254$\pm$4&257$\pm$8&299$\pm$12&270$\pm$5\\	
		\hline
	\end{tabular}
	\caption{Testing time (milliseconds) - comparison of DeepSSM and SSDNets on Sanyo, Hanergy, Solar and Exchange data sets.}
	\label{test_speed} 
\end{table*}

\subsection{Ablation Analysis}
To evaluate the effectiveness of the SSM part of SSDNet, we conducted additional analysis  - an ablation study. Table \ref{Ablation_Accuracy} shows the performance of SSDNet-LSTM, which is SSDNet with the Transformer part replaced by LSTM. Thus, we keep the SSM part for decomposition and prediction the same but replace the Transformer with LSTM for SSM parameters estimation. 

SSDNet-LSTM is compared with DeepAR \cite{DeepAR20} and DeepSSM \cite{DeepSSM18NIPS}. DeepAR uses LSTM to generate predictions directly and DeepSSM uses LSTM and Kalman filter to estimate the parameters of SSM, while SSDNet-LSTM can be considered as a combination of DeepAR and our proposed SSM or as a variation of DeepSSM by replacing the classic SSM with our proposed SSM without a Kalman filter.

Table \ref{Ablation_Accuracy} shows that SSDNet-LSTM outperforms both DeepAR and DeepSSM, indicating the effectiveness of the SSM part for decomposition and foresting. 
SSDNet-LSTM is also competitive with SSDNet which shows the robustness and flexibility of our fixed form SSM - its parameters can be estimated by different architectures - the Transformer, LSTM and other neural network models.

\subsection{Speed Analysis}
	
	We evaluate the training and testing time of the two SSDNet models (SSDNet and SSDNet-LSTM) and compare them with DeepSSM which also employs SSM. All models are trained on the same computer configuration; we report the average elapsed time and the standard deviation of 10 runs. Table \ref{train_speed} and \ref{test_speed} show the training time per epoch and testing time in milliseconds for the Sanyo, Hanergy, Solar and Exchange data sets. 
	No speed test was conducted for the Electricity data set as we report the accuracy results from their papers \cite{DeepSSM18NIPS,Logsparse19NIPS} and the values of the hyperparameters used are not known.

	Comparing the training speed, we can see that both SSDNet and SSDNet-LSTM, which do not use Kalman filter, are significantly faster than DeepSSM - the average training speed per epoch of DeepSSM is 390.48$\%$ and 458.24$\%$ times slower than SSDNet and SSDNet-LSTM respectively. In terms of testing speed, DeepSSM is the fastest, very closely followed by SSDNet-LSTM (7.68$\%$ slower) and then SSDNet. This is because the Kalman filter in DeerSSM is not required for inference and the architecture of LSTM used by DeepSSM is simpler than the Transformer used by SSDNet. 
	
	Overall, the training speed of SSDNet and SSDNet-LSTM is significantly improved due to the changes in SSM and their testing speed is competitive with DeepSSM.

\begin{figure}[!]
	\centering	  
	\subfigure[]{\includegraphics[width=\columnwidth]{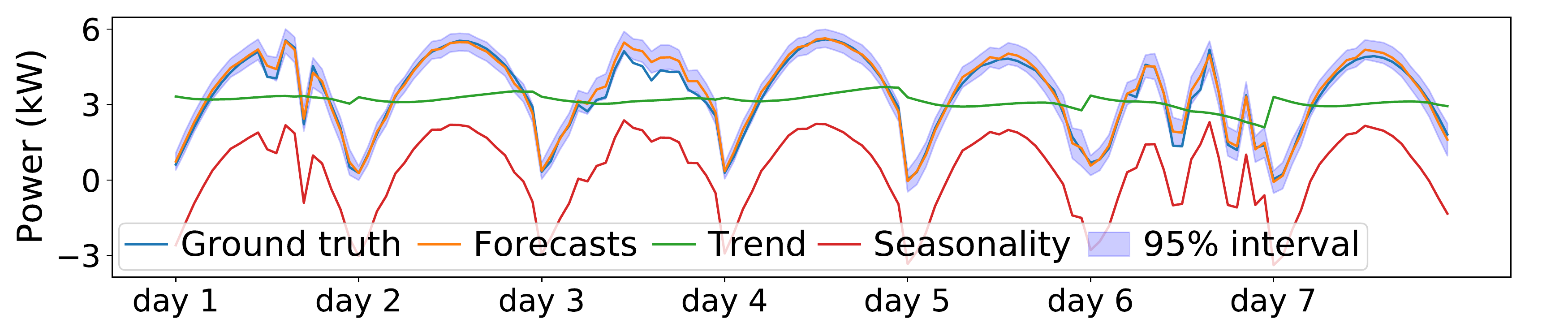}}		\subfigure[]{\includegraphics[width=\columnwidth]{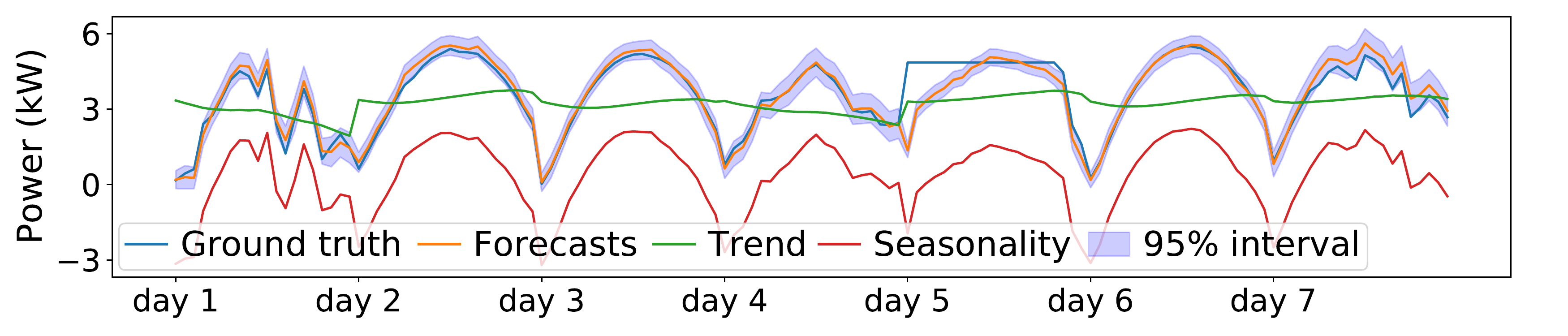}}	
	\subfigure[]{\includegraphics[width=\columnwidth]{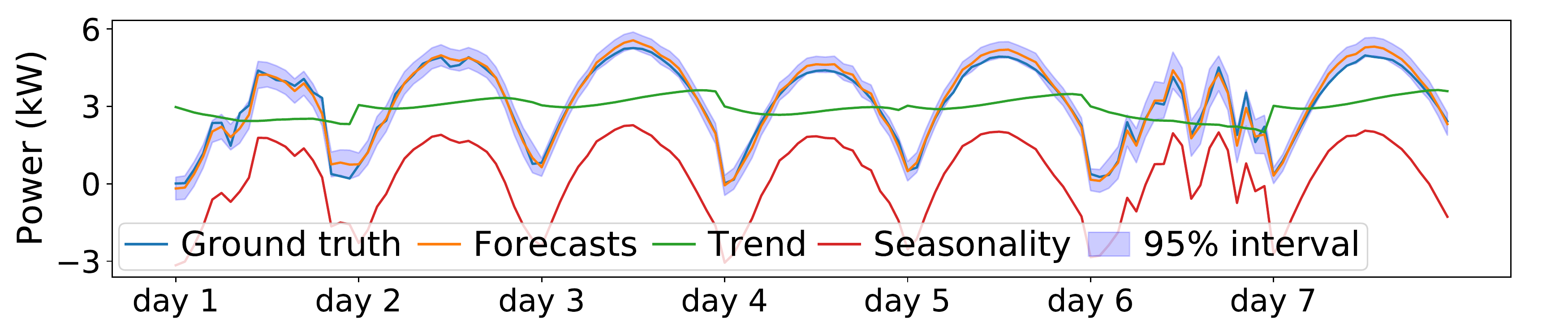}}		\subfigure[]{\includegraphics[width=\columnwidth]{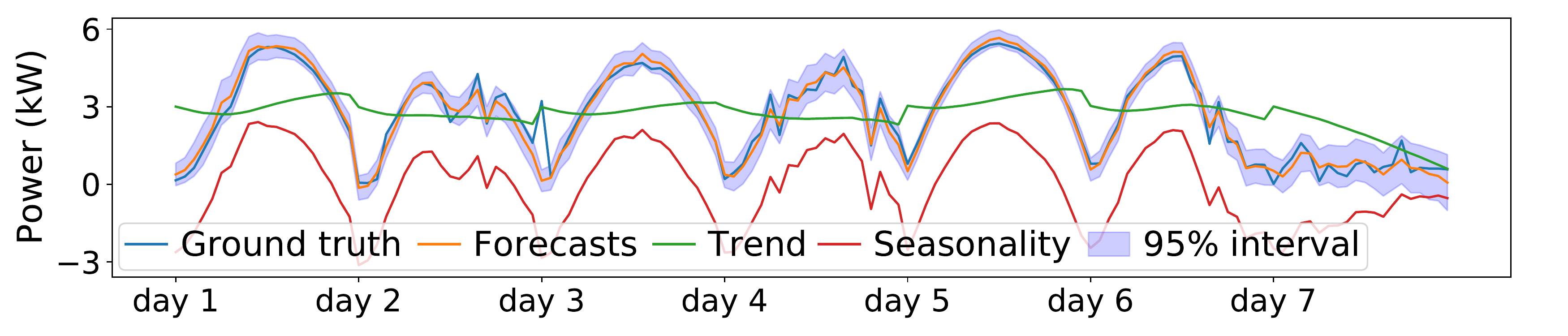}}	
	\subfigure[]{\includegraphics[width=\columnwidth]{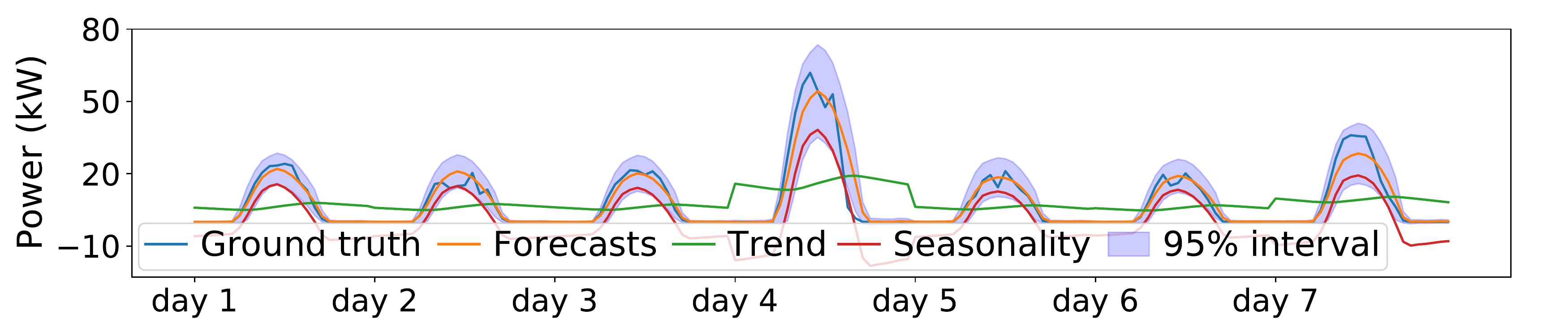}}			\subfigure[]{\includegraphics[width=\columnwidth]{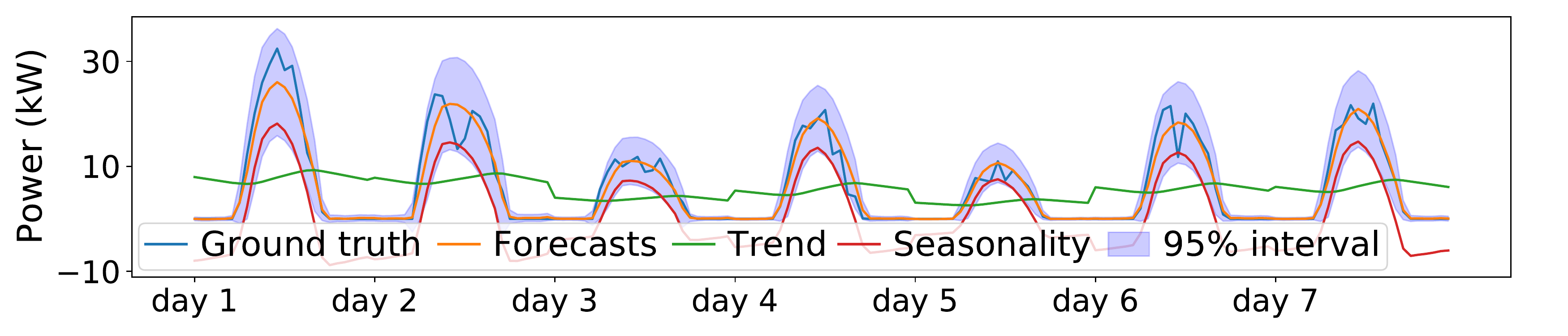}}		
	\subfigure[]{\includegraphics[width=\columnwidth]{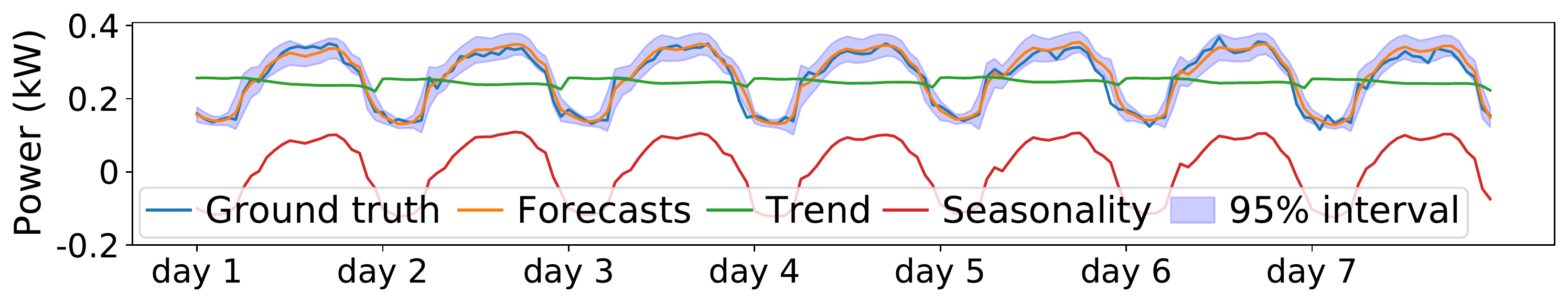}}	\subfigure[]{\includegraphics[width=\columnwidth]{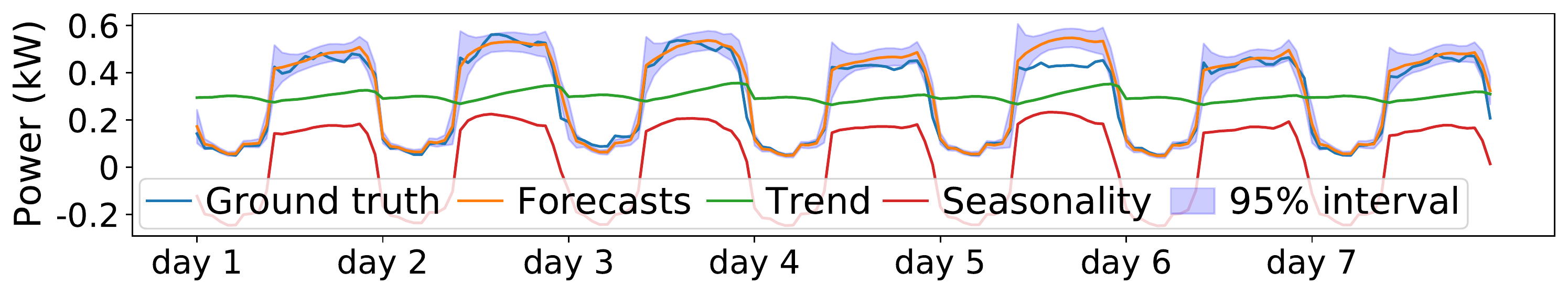}}
	\caption{Actual vs SSDNet predicted data with trend and seasonality components and 95$\%$ confidence  intervals: (a) and (b) - Sanyo; (c) and (d) - Hanergy; (e) and (f) - Solar; (g) and (h) - Electricity data sets}
	\label{Visualized_results}
\end{figure}
\begin{figure}[!ht]
	\centering	  
	\subfigure[]{\includegraphics[width=\columnwidth]{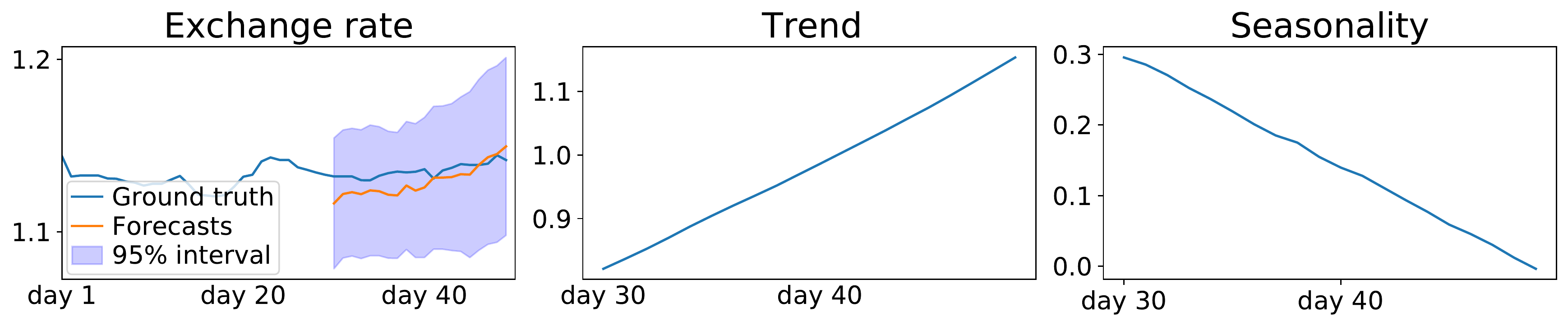}}		\subfigure[]{\includegraphics[width=\columnwidth]{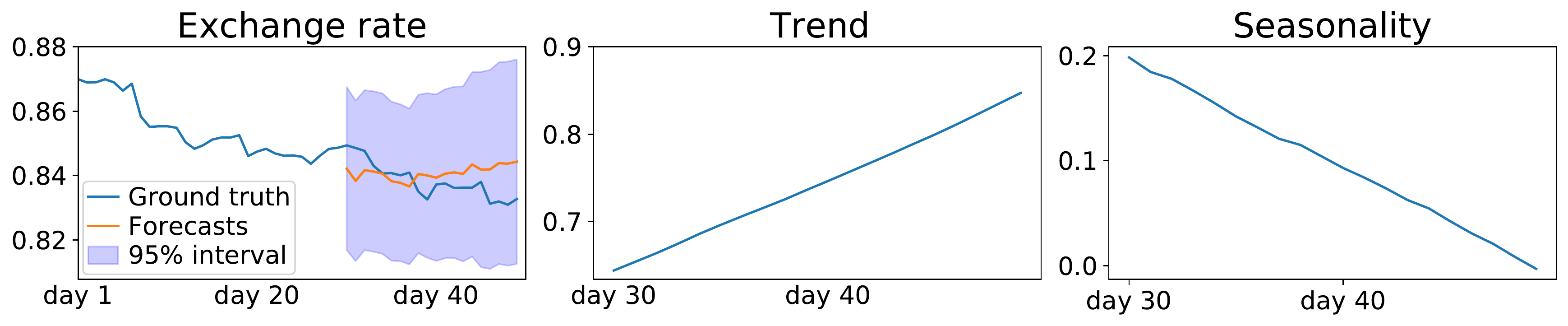}}	
	\caption{Actual vs SSDNet predicted data with trend and seasonality components and 95$\%$ confidence  intervals for Exchange data set}
	\label{Visualized_results_exchange}
\end{figure}
\begin{figure}[!ht]
	\centering	  
	\subfigure[]{\includegraphics[width=.49\columnwidth]{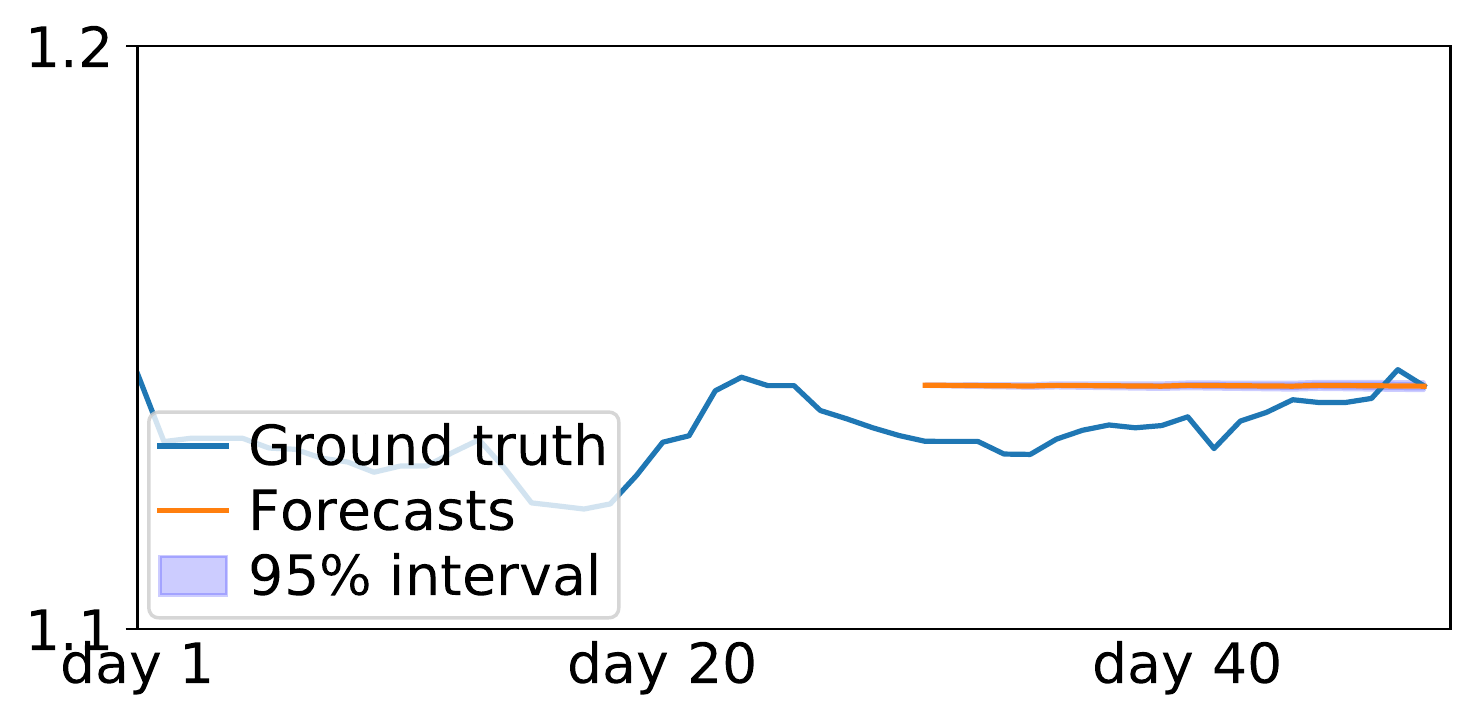}}		\subfigure[]{\includegraphics[width=.49\columnwidth]{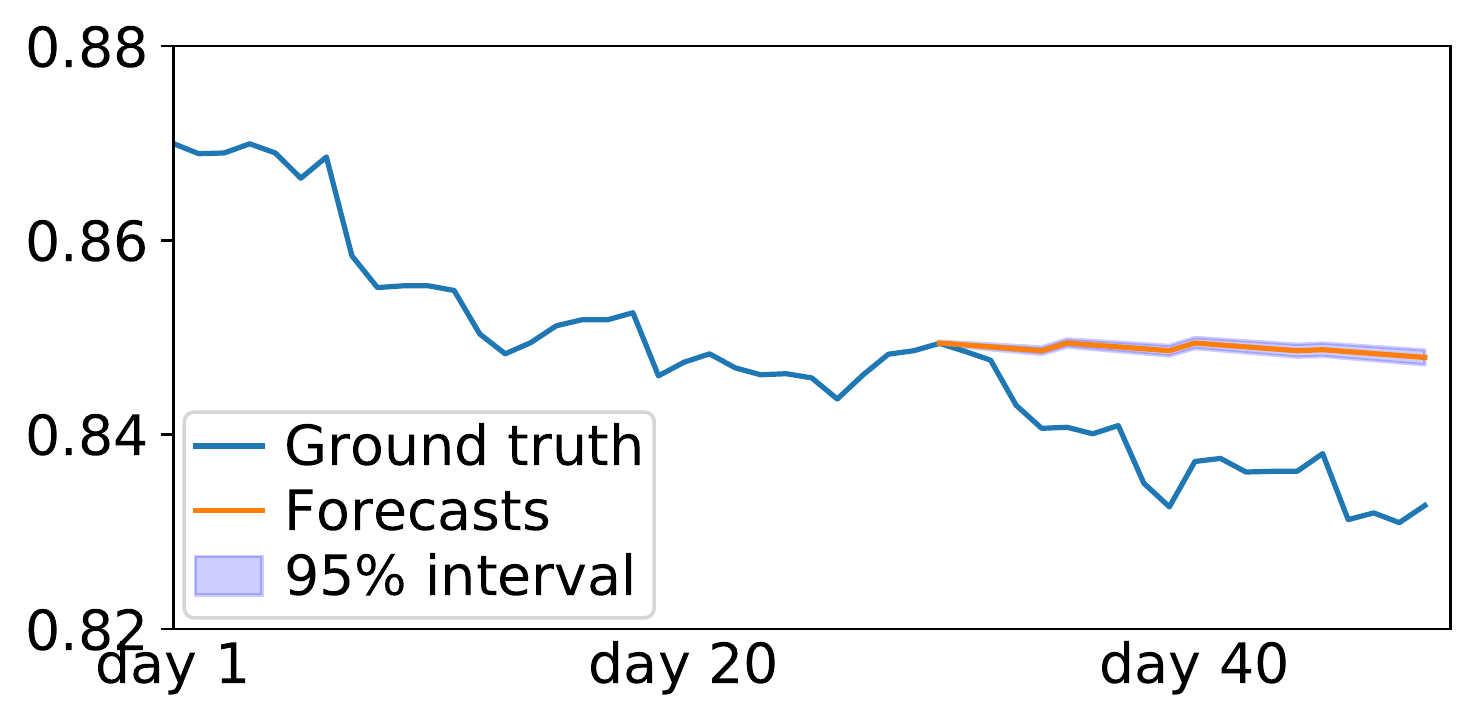}}	
	\caption{Actual vs SARIMAX predicted data and 95$\%$ confidence  intervals for Exchange data set}
	\label{Visualized_results_SARIMA}
\end{figure}
\begin{figure}[!ht]
	\centering	  
	\subfigure[]{\includegraphics[width=1\columnwidth]{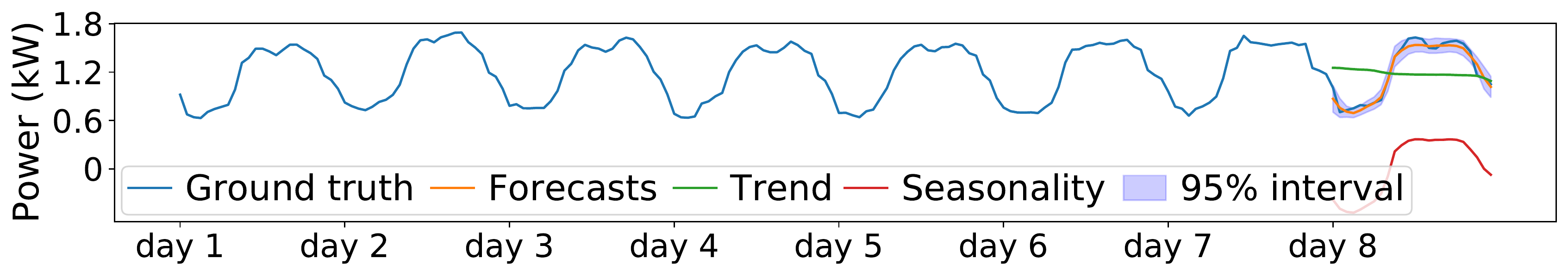}}
	\subfigure[]{\includegraphics[width=1\columnwidth]{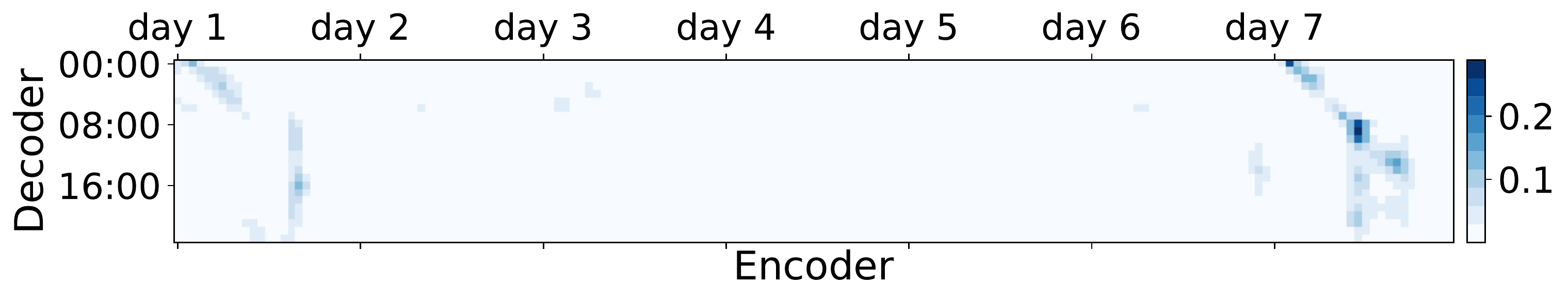}}		\subfigure[]{\includegraphics[width=1\columnwidth]{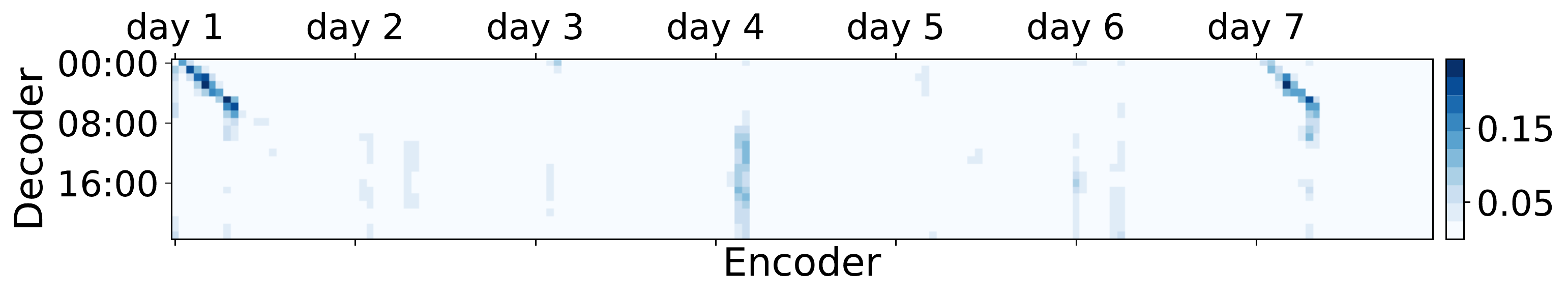}}
	\caption{SSDNet case study on Electricity data set (a) actual vs predicted data; (b) attention patterns of the first head in the last layer and (c) attention patterns of the second head in the last layer}
	\label{elect_atten}
\end{figure}

\subsection{Interpretability Analysis}
Fig. \ref{Visualized_results} presents the SSDNet output for a 7-day test set sample for the Sanyo, Hanergy, Solar and Electricity data sets, and Fig. \ref{Visualized_results_exchange} presents the SSDNet output for 2 test samples from the Exchange data set. We can see that the results are interpretable: the trend is smooth and monotonically decreasing or increasing, the seasonality shows regular and cyclic fluctuations (except for the Exchange data set which doesn't have seasonality). 

SSDNet can model a rapidly changing trend - for example, in Fig. \ref{Visualized_results} (d), the solar power generation drops rapidly on day 7 and SSDNet correctly outputs a declining trend.
SSDNet can not only learn smooth seasonal patterns but also has the ability to learn seasonal patters with small random fluctuations  (e.g. at days 1 and 6 of Fig. \ref{Visualized_results} (a)) due to the introduction of innovation term $c_t$.
For the Exchange data set which doesn't have clear seasonality, the peak-to-peak magnitude of the learned SSDNet seasonal component is small. 

Fig. \ref{elect_atten} (a) presents the results for 8 consecutive days from the test set for the Electricity data  - 7-day past history and 1-day forecasting output of SSDNet, and shows obvious daily seasonality. Fig. \ref{elect_atten} (b) and (c) show the attention patterns of two heads in the last multi-head attention layer of SSDNet.
The attention mappings show the importance of the previous time steps for predicting the future time steps.

The first head tends to attend to the first and seventh day before the predicted day, which is consistent with the weekly cycle of electricity data. The second head attends to almost all days. A possible explanation is that the first head mainly focuses on learning the seasonality, while the second head mainly focuses on learning the trend.

\section{Conclusion}

We presented SSDNet, a novel deep learning approach, for probabilistic and interpretable forecasting of time series data. 
SSDNet combines the advantages of deep learning with the interpretability of SSM models. SSDNet employs the Transformer architecture to learn the temporal patterns, extract latent components and estimate the parameters of SSM. It then applies SSM to generate the interpretable forecasting results with nonstationary trend and seasonality components. SSDNet also applies attention mechanism to visualize the important past time steps for the predicted future steps.

We evaluated the performance of SSDNet on five time series forecasting tasks. 
The results show that in terms of accuracy SSDNet outperformed the state-of-the-art deep learning models DeepSSM, DeepAR, LogSparse Transformer, Informer and N-BEATS, and the statistical models SARIMAX and Prophet. It was also able to provide interpretable results by showing clear trend and seasonality components. Our ablation study showed the effectiveness of the SSM part and that SSDNet with LSTM instead of Transformer also performed well. SSDNet was also much faster to train than DeepSSM. Hence, the results show that SSDNet is a promising method for time series forecasting. 

\bibliographystyle{IEEEtran}
\bibliography{main}

\end{document}